%% file: main.tex
\documentclass{article}
\usepackage{arxiv,times}
\input{math_commands.tex}

\usepackage{amsmath,amssymb,amsthm,mathtools}
\usepackage{graphicx}
\usepackage{hyperref}
\usepackage{url}
\usepackage{booktabs}
\usepackage{caption}
\usepackage{tabularx}
\usepackage{adjustbox}
\usepackage{longtable}
\usepackage{multirow}
\usepackage{array}
\usepackage[table]{xcolor}
\usepackage{float}
\usepackage{wrapfig}
\usepackage{placeins}
\usepackage{algorithm}
\usepackage{algpseudocode}

\newcommand{\tablefont}{\small}
\newcolumntype{Y}{>{\raggedright\arraybackslash}X}
\newcolumntype{Z}{>{\centering\arraybackslash}X}

\newtheorem{proposition}{Proposition}
\newtheorem{theorem}[proposition]{Theorem}
\newtheorem{lemma}[proposition]{Lemma}
\newtheorem{corollary}[proposition]{Corollary}
\theoremstyle{definition}
\newtheorem{assumption}{Assumption}
\theoremstyle{remark}
\newtheorem{remark}[proposition]{Remark}
\theoremstyle{plain}
\newcommand{\method}{\textsc{PMosFM}}
\newcommand{\norm}[1]{\left\lVert #1\right\rVert}
\newcommand{\pending}{\textcolor{gray}{--}}
\newcommand{\cmark}{\textcolor{green!45!black}{\ensuremath{\checkmark}}}
\newcommand{\xmark}{\textcolor{red!75!black}{\ensuremath{\times}}}

\title{PMosFM: Preconditioned Manifold Matching for One-Step Physics-Constrained Generation}
\hypersetup{
  pdftitle={PMosFM: Preconditioned Manifold Matching for One-Step Physics-Constrained Generation},
  pdfauthor={Zhangyong Liang, Haibin Ling},
  hypertexnames=false,
  hidelinks
}
\author{%
{\small\textbf{Zhangyong Liang}}\\
Department of Artificial Intelligence\\
Westlake University\\
\texttt{liangzhangyong@westlake.edu.cn}
\And
{\small\textbf{Haibin Ling}$^{*}$}\\
Department of Artificial Intelligence\\
Westlake University\\
\texttt{linghaibin@westlake.edu.cn}
}

\iclrfinalcopy
\begin{document}
\maketitle
\thispagestyle{pmosfmfirstpage}

\begin{abstract}
Physics-constrained generative models aim to generate physical fields that match a target distribution and satisfy prescribed constraints.
However, enforcing these constraints often increases sampling costs through iterative corrections or training costs through residual optimization and trajectory unrolling.
To address this issue, we introduce \textbf{P}reconditioned \textbf{M}anifold \textbf{o}ne-\textbf{s}tep \textbf{F}low \textbf{M}atching (\textbf{PMosFM}), a preconditioned manifold matching framework for one-step physics-constrained generation.
By encoding constraints in a manifold decoder, PMosFM learns transport in intrinsic coordinates without separate residual losses or terminal residual unrolling.
A geometric preconditioner rescales coordinates using the decoder-induced metric, while a regularized covariance transform approximately whitens the interpolation-state inputs.
A finite-interval objective couples velocity supervision with consistency between decoded endpoints in physical space.
We show that exact parameterization removes residual-induced Gauss--Newton curvature, that geometric and covariance effects separate in a local conditioning bound, and that physical flow-map error bounds endpoint distributional error.
Controlled ablations examine conditioning, and experiments evaluate optimizer-update time and memory footprint.
At inference, PMosFM uses one neural transport evaluation followed by physical decoding.
Experiments across benchmarks show lower training and sampling time than the multi-step baselines at comparable physical and distributional fidelity.
Code and datasets will be released publicly.
\end{abstract}

\section{Introduction}
\label{sec:introduction}

Flow matching (FM) learns probability transport from interpolation pairs and extends to function-valued physical fields~\citep{lipman2022flow,kerrigan2023functional}.
FNO~\citep{li2021fno} and DeepONet~\citep{lu2021deeponet} approximate deterministic PDE solution operators, while physics-informed training penalizes governing equation residuals~\citep{raissi2019physics}.
In contrast, conservation-aware models and differentiable solvers project solutions onto constrained manifolds~\citep{Derek2023,negiar2023hardconstraints}.
For uncertain parameters or partial observations, physics-constrained generation aims to recover distributions of physical states consistent with the prescribed constraints.

Physical information enters generative models through training losses or sampling operations.
Physics-informed diffusion models penalize residuals during training~\citep{Bastek2025}, while PCFT adapts pretrained flows through weak-form residuals~\citep{Tauberschmidt2025}.
CoCoGen and DiffusionPDE guide sampling with physical residuals~\citep{Jacobsen2025,huang2024diffusionpde}; D-Flow optimizes source noise through differentiable generation~\citep{ben2024d}.
Projected diffusion applies constraint projections during sampling~\citep{christopher2024constrained}.

\begin{figure*}[t]
  \centering
  \includegraphics[width=\textwidth]{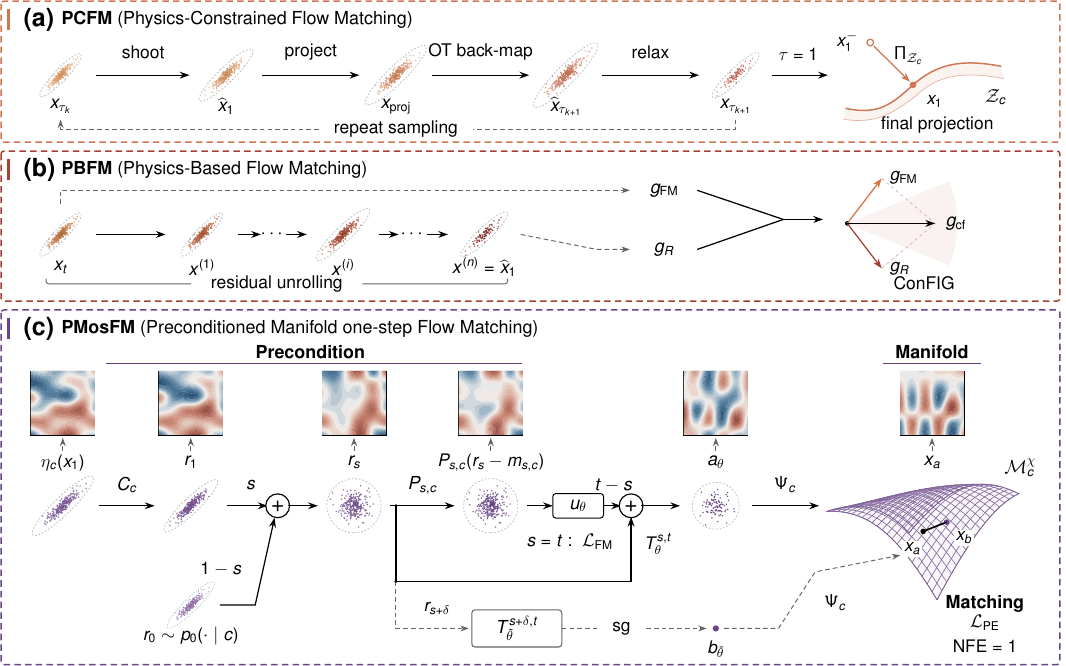}
\caption{
Overview of PMosFM and representative physics-constrained methods.
(a) PCFM projects a pre-correction endpoint $x_1^-$ onto the feasible set $\mathcal Z_c$ using $\Pi_{\mathcal Z_c}$.
(b) PBFM combines flow and residual gradients $g_{\mathrm{FM}}$ and $g_R$ into the ConFIG direction $g_{\mathrm{cf}}$.
(c) PMosFM applies precondition--manifold--matching, using $T_\theta^{0,1}$ for one-step transport and $\Psi_c$ for physical decoding.
}
 \label{fig:pmosfm-framework}
\end{figure*}

Existing methods face different computational bottlenecks: inference-time projection vs. training-time unrolling.
ECI~\citep{Cheng2025} and PCFM~\citep{utkarsh2025pcfm} can reuse pretrained models, but constraint handling remains in the online sampling loop.
PBFM~\citep{baldan2025pbfm} avoids physics-specific correction at inference by adding residual-gradient computations and terminal unrolling during training.
Studies of physics-informed learning also identify gradient imbalance and ill-conditioning as optimization difficulties~\citep{krishnapriyan2021failure,wang2020pathologies,cao2025illconditioning}.
Training-time enforcement shifts physics computation offline, while residual optimization and iterative learned transport remain.
This raises a natural question: \emph{\textbf{can physical constraint satisfaction and transport be unified to reduce both training and sampling costs?}}

Recent work reduces iterative generative transport through path straightening~\citep{liu2022rectified}, consistency~\citep{song2023consistency}, or direct finite-interval learning~\citep{han2026wflow}.
Physics-informed distillation further enables one-step physics-constrained generation~\citep{zhang2025piddm}.
However, one-step physical generation must address not only transport depth, but also the representation of feasible states and the metric in which finite-interval errors are measured.
Meanwhile, manifold generative models extend transport to non-Euclidean
state spaces~\citep{mathieu2020riemanniancnf,debortoli2022riemannian,zhong2026rmf}, while preconditioned flow matching improves
optimization under anisotropic interpolation statistics ~\citep{ahamed2026preconditioned}.
Feasible coordinates expose two distinct conditioning factors: decoder-induced metric distortion and interpolation-state covariance.

In this work, we introduce \textbf{P}reconditioned \textbf{M}anifold \textbf{o}ne-\textbf{s}tep \textbf{F}low \textbf{M}atching (\textbf{PMosFM}) to address these challenges via the proposed \emph{precondition--manifold--matching} strategy.
For an exact residual parameterization on the valid domain, PMosFM removes a separate residual penalty and terminal residual unrolling for the encoded constraints.
Manifold and interpolation-state preconditioners condition feasible coordinates and network inputs. A finite-interval objective couples velocity supervision with endpoint consistency in physical space.
Together, these components define a one-step generator evaluated by one finite-interval neural map followed by physical decoding (Fig.~\ref{fig:pmosfm-framework}).

We summarize the main contributions as follows:
\begin{itemize}
 \item \textbf{One-step physics-constrained generation.} We formulate finite-interval transport in feasible coordinates with one-step evaluation followed by physical decoding.
  For an exact parameterization on the valid domain, the representation enforces the encoded residual directly, replacing separate training penalties and repeated unrolling.
 \item \textbf{Precondition--manifold--matching strategy.} We formulate physical endpoint matching with geometric and input-covariance preconditioning. Velocity supervision at $s=t$ learns the interpolation field, while endpoint consistency at $s<t$ trains finite-interval transport in feasible coordinates.
 \item \textbf{Feasibility, conditioning, and transport.} We show that exact residual coordinates remove the Gauss--Newton curvature contributed by the encoded constraint, while geometric and statistical preconditioning control the conditioning of the remaining transport. 
 We further bound endpoint distribution error by physical flow-map error and distinguish constraint satisfaction from matching the target law.  
 \item \textbf{Efficient training and one-step sampling.} PMosFM reduces training time per update compared with training with residual unrolling and achieves faster convergence with the preconditioned residual-manifold formulation.
 Across five benchmarks, PMosFM outperforms the multi-step baselines in distributional accuracy and sampling speed, using a single network evaluation followed by physical decoding.
\end{itemize}

\section{Related Work}
\label{sec:related-work}

\paragraph{Physics-constrained generative modeling.}
Physics enters training through residual losses or virtual observables \citep{Bastek2025,rixner2021virtual}.
PBFM couples residual optimization, conflict-free gradients, and unrolling \citep{baldan2025pbfm,liu2025config}, whereas PCFT applies weak-form fine-tuning \citep{Tauberschmidt2025}.
Related approaches include residual-gradient conditioning \citep{shu2023physics}, diffusion posterior sampling \citep{chung2022dps}, and manifold-constrained diffusion \citep{chung2022manifold}.
During sampling, CoCoGen and DiffusionPDE impose physical residuals \citep{Jacobsen2025,huang2024diffusionpde}, while D-Flow optimizes source noise through differentiable generation \citep{ben2024d}.
Hard-constraint methods include ECI and PCFM \citep{Cheng2025,utkarsh2025pcfm}, while projected diffusion \citep{christopher2024constrained} and chance-constrained \citep{liang2025chance} flow matching provide alternative constraint-handling strategies.

\paragraph{One-step generative transport.}
FM learns velocity fields for finite-dimensional and function-valued data~\citep{lipman2022flow,kerrigan2023functional}.
Stochastic interpolants unify flow and diffusion objectives~\citep{albergo2023stochastic}, while minibatch optimal transport modifies source--target coupling~\citep{tong2024minibatchot}.
Rectified flows simplify transport paths~\citep{liu2022rectified}; consistency and shortcut models support one-step generation~\citep{song2023consistency,frans2024shortcut}.
SoFlow combines flow matching with solution consistency for one-step generation~\citep{luo2026soflow}.
InstaFlow distills a straightened flow~\citep{liu2024instaflow}, whereas MeanFlow learns finite-interval average velocities~\citep{geng2025mean}.
W-Flow formulates one-step generation via distributional gradient flows~\citep{han2026wflow}, whereas PIDDM enforces PDE constraints through post-hoc distillation for one-step generation~\citep{zhang2025piddm}.

\paragraph{Manifold and preconditioning.}
Riemannian continuous flows, score models, and manifold ODEs incorporate prescribed non-Euclidean geometry~\citep{mathieu2020riemanniancnf,debortoli2022riemannian,lou2020manifoldode}.
Manifold-aware FM extends vector-field regression to such spaces~\citep{chen2023flow}, and Riemannian MeanFlow learns finite-interval transport through tangent-space alignment~\citep{zhong2026rmf}.
Preconditioned Flow Matching addresses interpolation-covariance anisotropy through invertible transformations~\citep{ahamed2026preconditioned}.

Table~\ref{tab:method-capabilities} compares PMosFM with representative baselines, highlighting the preconditioned
manifold formulation for one-step physics-constrained generation.

\begin{table*}[t]
\setcitestyle{numbers,square}
\centering
\caption{Comparison between \method{} and other physics-constrained generative models.}
\label{tab:method-capabilities}
\tablefont
\setlength{\tabcolsep}{3.0pt}
\renewcommand{\arraystretch}{1.05}
\resizebox{\textwidth}{!}{%
\begin{tabular}{@{}>{\raggedright\arraybackslash}p{0.18\textwidth}*{8}{c}@{}}
\toprule
\textbf{Method} &
\shortstack{\textbf{Physics at}\\\textbf{training}} &
\shortstack{\textbf{Hard}\\\textbf{constraints}} &
\shortstack{\textbf{Gradient-free}\\\textbf{inference}} &
\shortstack{\textbf{Complex}\\\textbf{constraints}} &
\shortstack{\textbf{No manual physics--}\\\textbf{fidelity balancing}} &
\shortstack{\textbf{One step}\\\textbf{sampling}} &
\shortstack{\textbf{Residual-factorized}\\\textbf{physics training}} &
\shortstack{\textbf{Spectral/Jacobian}\\\textbf{preconditioning}} \\
\midrule
FFM~\citep{kerrigan2023functional} & \xmark & \cmark & \cmark & \xmark & \xmark & \xmark & \xmark & \xmark \\
FM-OT~\citep{lipman2022flow} & \xmark & \xmark & \cmark & \xmark & \xmark & \xmark & \xmark & \xmark \\
CoCoGen~\citep{Jacobsen2025} & \xmark & \cmark & \cmark & \cmark & \xmark & \xmark & \xmark & \xmark \\
DiffusionPDE\citep{huang2024diffusionpde} & \xmark & \xmark & \cmark & \cmark & \xmark & \xmark & \xmark & \xmark \\
PIDM~\citep{Bastek2025} & \cmark & \xmark & \cmark & \cmark & \xmark & \xmark & \xmark & \xmark \\
D-Flow~\citep{ben2024d} & \xmark & \xmark & \cmark & \cmark & \xmark & \xmark & \xmark & \xmark \\
ECI~\citep{Cheng2025} & \xmark & \cmark & \cmark & \xmark & \xmark & \xmark & \xmark & \xmark \\
PCFM~\citep{utkarsh2025pcfm} & \xmark & \cmark & \cmark & \cmark & \xmark & \xmark & \xmark & \xmark \\
CCFM~\citep{liang2025chance} & \xmark & \cmark & \cmark & \cmark & \xmark & \xmark & \xmark & \xmark \\
PCFT~\citep{Tauberschmidt2025} & \cmark & \xmark & \cmark & \cmark & \xmark & \xmark & \xmark & \xmark \\
PBFM~\citep{baldan2025pbfm} & \cmark & \xmark & \cmark & \cmark & \xmark & \xmark & \xmark & \xmark \\
\textbf{PMosFM} & \cmark & \cmark & \cmark & \cmark & \cmark & \cmark & \cmark & \cmark \\
\bottomrule
\end{tabular}
}
\end{table*}

\section{Methodology}
\label{sec:methodology}
\subsection{Preliminaries}
\label{sec:preliminaries}

Let $\mathcal X_h\subseteq\mathbb R^d$ be the discrete state space on mesh $h$, with state $x\in\mathcal X_h$ and $d$ degrees of freedom. 
And let $c$ collect the conditioning variables, including manifold, forcing, boundary data, observations, and physical time when applicable.
We use $t\in[0,1]$ for sampling time, distinct from the physical time when included in $c$.
For the physical residual $R_h(\cdot,c)$, define the feasible set
\begin{equation}
\mathcal Z_c=\{x\in\mathcal X_h:R_h(x,c)=0\}.
\label{eq:feasible-set}
\end{equation}
We model a regular component of $\mathcal Z_c$; the residual specifies feasible support and a target law specifies probability mass on it. Physical errors are measured by $\|e\|_{M_h}^2=e^\top M_h e$, where $M_h$ is a symmetric positive-definite quadrature matrix.

\subsection{Physics-based flow matching}
\label{sec:design-overview}

Let $\nu_0$ be the ambient source law and $p_{\mathrm{data}}(\cdot\mid c)$ the conditional data law. For an endpoint pair $x_0\sim\nu_0$ and $x_1\sim p_{\mathrm{data}}(\cdot\mid c)$, consider ambient flow matching along
\begin{equation}
    x_\tau=(1-\tau)x_0+\tau x_1,
    \qquad
    w_x=x_1-x_0,
\end{equation}
where $\tau\in[0,1]$. A neural velocity field $v_\phi$, with trainable parameters $\phi$, is fitted to $w_x$ using
\begin{equation}
\label{eq:ambient_fm}
    \mathcal L_{\mathrm{FM}}^{\mathrm{amb}}
    =
    \mathbb E\!\left[
    \frac{1}{2d}
    \left\|
    v_\phi(x_\tau,\tau,c)-w_x
    \right\|_2^2
    \right],
\end{equation}
The expectation averages over sampled endpoint pairs and interpolation times, and over conditions when training jointly across $c$.
Ambient interpolation may leave the feasible set, so existing methods impose physical constraints during training~\citep{baldan2025pbfm} or sampling~\citep{utkarsh2025pcfm}.
Residual penalties can make endpoint regression ill-conditioned (Lemma~\ref{lem:app-hessian-decomposition} in Appendix~\ref{app:large-step-theory}; mesh-refinement analysis in Appendix~\ref{app:normal-tangent-proof}).

\subsection{PMosFM: Preconditioned Manifold one-step Flow Matching}
\label{sec:pmosfm}

PMosFM learns one-step transport on a represented feasible manifold, applying coordinate preconditioning and endpoint matching in physical space.
Fig.~\ref{fig:pmosfm-manifolds} compares the methods in a
\begin{wrapfigure}{r}{0.46\textwidth}
 \centering
 \includegraphics[width=\linewidth]{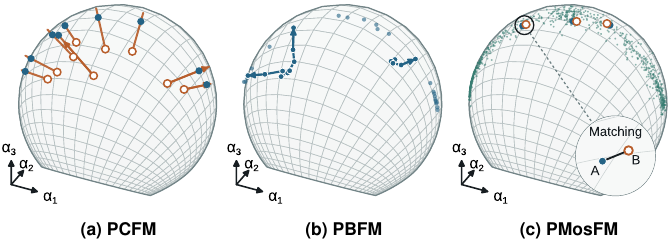}
 \caption{Manifold coordinates and endpoint matching.
 Schematic comparison of PCFM, PBFM, and PMosFM.}
 \label{fig:pmosfm-manifolds}
\end{wrapfigure}
schematic embedding space:
\textbf{(a)}~PCFM alternates full-space integration with iterative projection onto constraints;
\textbf{(b)}~PBFM steers full-space paths with residual gradients unrolling; and
\textbf{(c)}~PMosFM decodes two endpoint estimates and compares them on the represented residual manifold (inset).

\paragraph{Feasible parameterization.}

Let $y\in\mathcal Y_c\subset\mathbb R^m$ collect the $m$ independent coordinates on the represented feasible component.
On the represented regular region, assume a differentiable local parameterization $\chi_c:\mathcal Y_c\to\mathcal X_h$ whose Jacobian has $m$ linearly independent columns and satisfies
\begin{equation}
 R_h(\chi_c(y),c)=0
 \qquad\text{for every }y\in\mathcal Y_c.
 \label{eq:exact-chart}
\end{equation}

The represented residual manifold is $\mathcal M_c^\chi=\chi_c(\mathcal Y_c)\subseteq\mathcal Z_c$.
Every decoded state satisfies the encoded constraints within the valid coordinate domain.
Lemma~\ref{lem:app-tangent-space} in Appendix~\ref{app:manifold-theory} gives the differential characterization.
The decoder and encoded-residual evaluator use the same discrete operators, boundary conditions, and sign conventions.
The constant-rank and coordinate assumptions are stated in Appendix~\ref{app:mathematical-setting}.
The encoder $\eta_c:\mathcal M_c^\chi\to\mathcal Y_c$ returns coordinates satisfying $\chi_c(\eta_c(x))=x$ for represented data.
For data outside the represented set, encoding followed by decoding produces a measurable physical-metric projection (Appendix~\ref{app:encoding-error}).
We distinguish parameterization error from velocity-regression error.

Let $\mu_c\in\mathbb R^m$ be a coordinate center and $C_c\in\mathbb R^{m\times m}$ a fixed invertible transform that rescales physical error across coordinate directions.
The coordinate and physical decoder are
\begin{equation}
 r=C_c(y-\mu_c),
 \qquad
 \Psi_c(r)=\chi_c\!\left(\mu_c+C_c^{-1}r\right).
 \label{eq:preconditioned-coordinate}
\end{equation}
The valid $r$-domain is $C_c(\mathcal Y_c-\mu_c)$, and the results below assume that interpolation states and predicted endpoints remain in this domain.
For a data sample $x_1\sim p_{\mathrm{data}}(\cdot\mid c)$, set $r_1=C_c(\eta_c(x_1)-\mu_c)$.
PMosFM pairs the encoded target $r_1$ with a source sample $r_0\sim p_0(\cdot\mid c)$.
The decoded target sample is $\Psi_c(r_1)$.
Appendices~\ref{app:encoding-error} and~\ref{app:coarea-law} cover projection error and the co-area law under hard conditioning, respectively.

\paragraph{Manifold and preconditioning.}
\label{sec:metric-aware-preconditioning}

Geometric preconditioning rescales the decoder-induced pullback metric in feasible coordinates to reduce local metric distortion, as illustrated in Fig.~\ref{fig:pmosfm-preconditioned-geometry}.

\begin{wrapfigure}{r}{0.48\textwidth}
    \centering
    \includegraphics[width=0.98\linewidth]{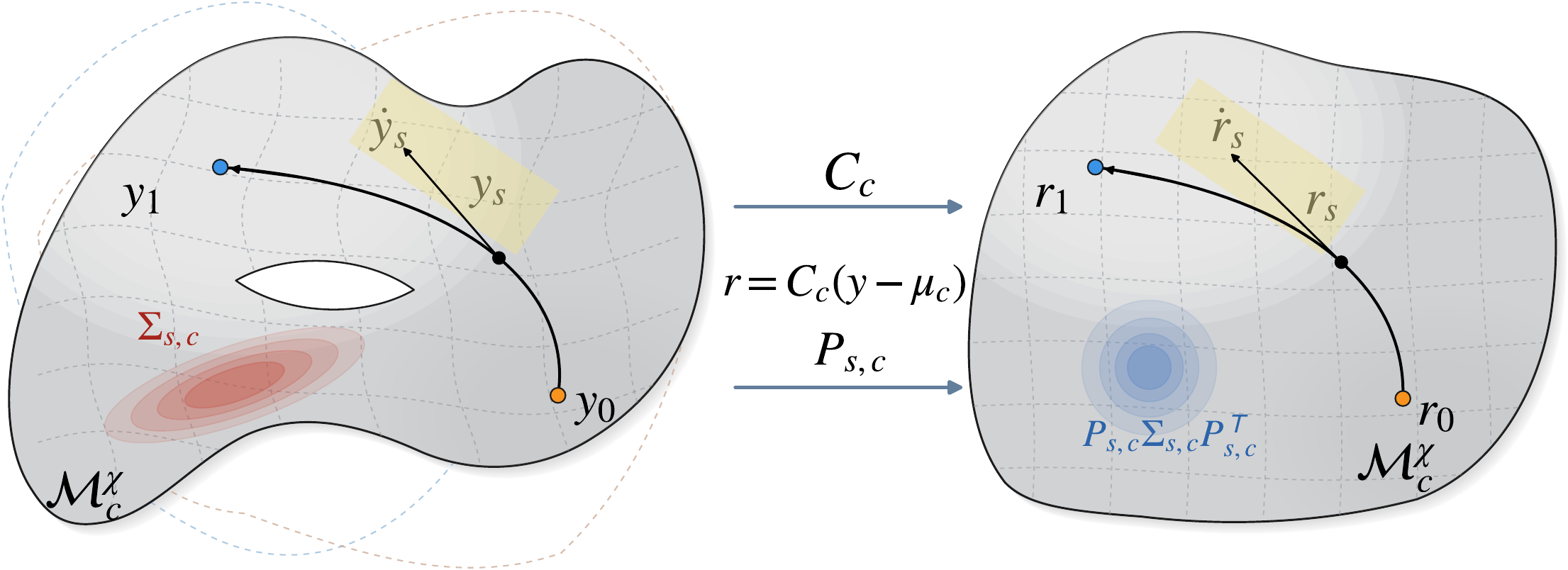}
    \caption{Preconditioned residual geometry.
    $y_i=\mu_c+C_c^{-1}r_i$; dots denote $s$-derivatives, with $\dot r_s=w$ and $\dot y_s=C_c^{-1}w$.}
    \label{fig:pmosfm-preconditioned-geometry}
\end{wrapfigure}

Let $J_\chi=D_y\chi_c$ denote the decoder Jacobian.
The physical error metric in the original coordinates is $G_c(y)=J_\chi(y,c)^\top M_hJ_\chi(y,c)$.
Under the coordinate transformation in Eq.~\ref{eq:preconditioned-coordinate}, the metric in the $r$ coordinates becomes
\begin{equation}
 \widetilde G_c(r)=C_c^{-\top}G_c(y)C_c^{-1}.
  \label{eq:pullback-physical-metric}
\end{equation}

The transformation $C_c$ rescales the pullback metric $\widetilde G_c$, while the endpoint loss in Eq.~\ref{eq:physical-endpoint-loss} is evaluated directly in physical space.
For an interpolation pair $(r_0,r_1)$, define
\begin{equation}
 \left\{
 \begin{aligned}
 r_s&=(1-s)r_0+s r_1,\\
 w&=r_1-r_0,\qquad 0\leq s\leq1.
 \end{aligned}
 \right.
 \label{eq:interpolation-pair}
\end{equation}
For each condition and generation time, define the input statistics and regularized preconditioner
\begin{equation}
 m_{s,c}=\mathbb E[r_s\mid c],
 \quad
 \Sigma_{s,c}=\mathop{\mathrm{Cov}}\nolimits(r_s\mid c),
 \quad
 P_{s,c}=(\Sigma_{s,c}+\varepsilon_P I)^{-1/2},
 \label{eq:input-preconditioner}
\end{equation}
where $m_{s,c}$ and $\Sigma_{s,c}$ are the conditional mean and covariance of $r_s$, $\varepsilon_P>0$ regularizes the covariance, and $I$ is the $m\times m$ identity.

The fixed $C_c$ changes the coordinates used by $\Psi_c$, whereas $P_{s,c}$ rescales only the network inputs; neither changes the encoded physical target law (Appendix~\ref{app:preconditioning-theory}).

\paragraph{Physical endpoint matching.}
\label{sec:intrinsic-one-step}

PMosFM couples velocity supervision with endpoint consistency in physical space.
The network $u_\theta$, with trainable parameters $\theta$, predicts an $m$-dimensional displacement rate in the $r$ coordinates from preconditioned inputs.
For $0\leq s\leq t\leq1$, define the map from generation time $s$ to $t$ by
\begin{equation}
 T_\theta^{s,t}(r,c)
 =r+(t-s)u_\theta\!\left(P_{s,c}(r-m_{s,c}),s,t,c\right).
 \label{eq:two-time-map}
\end{equation}
The losses below are averaged over sampled pairs $(r_0,r_1)$, times $(s,t,\delta)$ where used, and conditions $c$ when training jointly.
The $s=t$ branch retains ordinary flow matching:
\begin{equation}
 \mathcal L_{\mathrm{FM}}
 =\mathbb E\left[
 \frac{1}{2m}
 \norm{u_\theta\!\left(P_{s,c}(r_s-m_{s,c}),s,s,c\right)-w}_2^2
 \right].
 \label{eq:diagonal-fm-loss}
\end{equation}
This term supervises the velocity of the interpolation path and retains the flow-matching anchor.
For the $s<t$ endpoint-consistency branch, let $\delta$ be the separation between the two starting times, with $0<\delta<t-s$.
Set $r_{s+\delta}=r_s+\delta w$.
The two endpoint estimates are
\begin{equation}
 a_\theta=T_\theta^{s,t}(r_s,c),
 \qquad
 b_{\bar\theta}=T_{\bar\theta}^{s+\delta,t}(r_{s+\delta},c),
 \label{eq:two-time-endpoints}
\end{equation}
where $\bar\theta=\mathop{\mathrm{sg}}\nolimits(\theta)$.
The stopped-gradient branch takes a current-parameter copy; gradients are stopped through this branch.
Because both branches predict the endpoint at time $t$, PMosFM compares their decoded physical states:
\begin{equation}
 \mathcal L_{\mathrm{PE}}
 =\mathbb E\left[
 \frac{\norm{\Psi_c(a_\theta)-\mathop{\mathrm{sg}}\nolimits[\Psi_c(b_{\bar\theta})]}_{M_h}^{2}}
 {2m\,\delta(t-s)}
 \right].
 \label{eq:physical-endpoint-loss}
\end{equation}

In Algorithm~\ref{alg:pmosfm}, lowercase $\ell_{\mathrm{FM}}$ and $\ell_{\mathrm{PE}}$ denote minibatch estimates of the corresponding population losses.
Within the valid coordinate domain, the decoded endpoints satisfy the encoded residual and are compared in physical space using $M_h$.

\begin{wrapfigure}{r}{0.48\linewidth}
\begin{minipage}{\linewidth}
\begin{algorithm}[H]
\caption{\textsc{PMosFM}: training and sampling.}
\label{alg:pmosfm}
\small
\begin{algorithmic}[1]
\Require Data $p_{\mathrm{data}}(x\mid c)$, source $p_0(r\mid c)$, model $u_\theta$, encoder--decoder $(\eta_c,\chi_c)$, metric $M_h$, $\gamma$.
\State Estimate $(\mu_c,C_c,m_{s,c},P_{s,c})$ \Comment{Precondition}
\While{not converged}
    \State Select $c$; draw $(s,t,\delta)$
    \State $x_1\sim p_{\mathrm{data}}(\cdot\mid c),\ r_0\sim p_0(\cdot\mid c)$
    \State $r_1\leftarrow C_c(\eta_c(x_1)-\mu_c)$ \Comment{Feasible manifold}
    \State $r_s\leftarrow(1-s)r_0+sr_1,\quad w\leftarrow r_1-r_0$
    \State $r_{s+\delta}\leftarrow r_s+\delta w$
    \State Evaluate $\ell_{\mathrm{FM}}$ in Eq.~\ref{eq:diagonal-fm-loss} \Comment{Velocity anchor}
    \State $\bar\theta\leftarrow\mathop{\mathrm{sg}}\nolimits(\theta)$
    \State $a_\theta\leftarrow T_\theta^{s,t}(r_s,c)$, $b_{\bar\theta}\leftarrow T_{\bar\theta}^{s+\delta,t}(r_{s+\delta},c)$
    \State Evaluate $\ell_{\mathrm{PE}}$ in Eq.~\ref{eq:physical-endpoint-loss} \Comment{Manifold matching}
    \State $\theta\leftarrow\mathop{\mathrm{Update}}\nolimits\!\left[\gamma\ell_{\mathrm{FM}}+(1-\gamma)\ell_{\mathrm{PE}}\right]$
\EndWhile
\State $r_0\sim p_0(\cdot\mid c),\quad \widehat r_1\leftarrow T_\theta^{0,1}(r_0,c)$ \Comment{One-step}
\State $\widehat x_1\leftarrow\Psi_c(\widehat r_1)$ \Comment{Physical decode}
\State \Return $\widehat x_1$
\end{algorithmic}
\end{algorithm}
\end{minipage}
\end{wrapfigure}

The complete loss is
\begin{equation}
 \mathcal L_{\mathrm{PMosFM}}
 =\gamma\mathcal L_{\mathrm{FM}}+(1-\gamma)\mathcal L_{\mathrm{PE}}.
 \label{eq:pmosfm-loss}
\end{equation}
where the weight $\gamma\in[0,1]$ balances velocity supervision and physical-endpoint consistency.

Algorithm~\ref{alg:pmosfm} summarizes the core PMosFM training and one-step sampling procedure.
Algorithms~\ref{alg:pmosfm-training-full} and~\ref{alg:pmosfm-sampling-full} in Appendix~\ref{app:pmosfm-algorithms} provide the complete training and sampling procedures.

\paragraph{Feasibility and conditioning.}
\label{sec:training-inference}

PMosFM evaluates the neural transport map, $\widehat r_1=T_\theta^{0,1}(r_0,c)$, and applies the explicit decoder $\widehat x_1=\Psi_c(\widehat r_1)$.
\begin{proposition}[Exact feasibility]
\label{prop:residual-annihilation}
For fixed $c$ and $\widehat r_1$ in the coordinate domain, Eq.~\ref{eq:exact-chart} gives
\begin{equation}
R_h(\widehat x_1,c)=0.
\label{eq:normal-annihilation}
\end{equation}
The encoded residual vanishes on the valid domain, contributing no Gauss--Newton curvature (see Appendix~\ref{app:exact-manifold-proof}).
\end{proposition}

Consider a local velocity model linear in the preconditioned inputs, with coefficient matrix $A\in\mathbb R^{m\times m}$.
Let $H_A^{\mathrm{GN}}$ denote the Gauss--Newton matrix of the decoded-endpoint loss with respect to $A$, at fixed $s<t$, $c$, $\delta$, and stopped-gradient target.
Assume that the input covariance $P_{s,c}\Sigma_{s,c}P_{s,c}^{\top}$ is positive definite and that the eigenvalues of $\widetilde G_c$ lie in $[\alpha,\beta]$ throughout the evaluated endpoint region, with $0<\alpha\leq\beta$.
Proposition~\ref{prop:app-local-pe-conditioning} in Appendix~\ref{app:preconditioning-theory} gives
\begin{equation}
\kappa(H_A^{\mathrm{GN}})
\leq
\frac{\beta}{\alpha}\,
\kappa\!\left(P_{s,c}\Sigma_{s,c}P_{s,c}^\top\right),
\label{eq:method-local-conditioning}
\end{equation}
where $\kappa$ denotes the condition number.
The bound separates decoder-geometry and input-covariance effects; the appendices treat approximate decoders, endpoint error, and multiple local charts.

\begin{wrapfigure}{r}{0.53\textwidth}
\centering
\includegraphics[width=\linewidth]{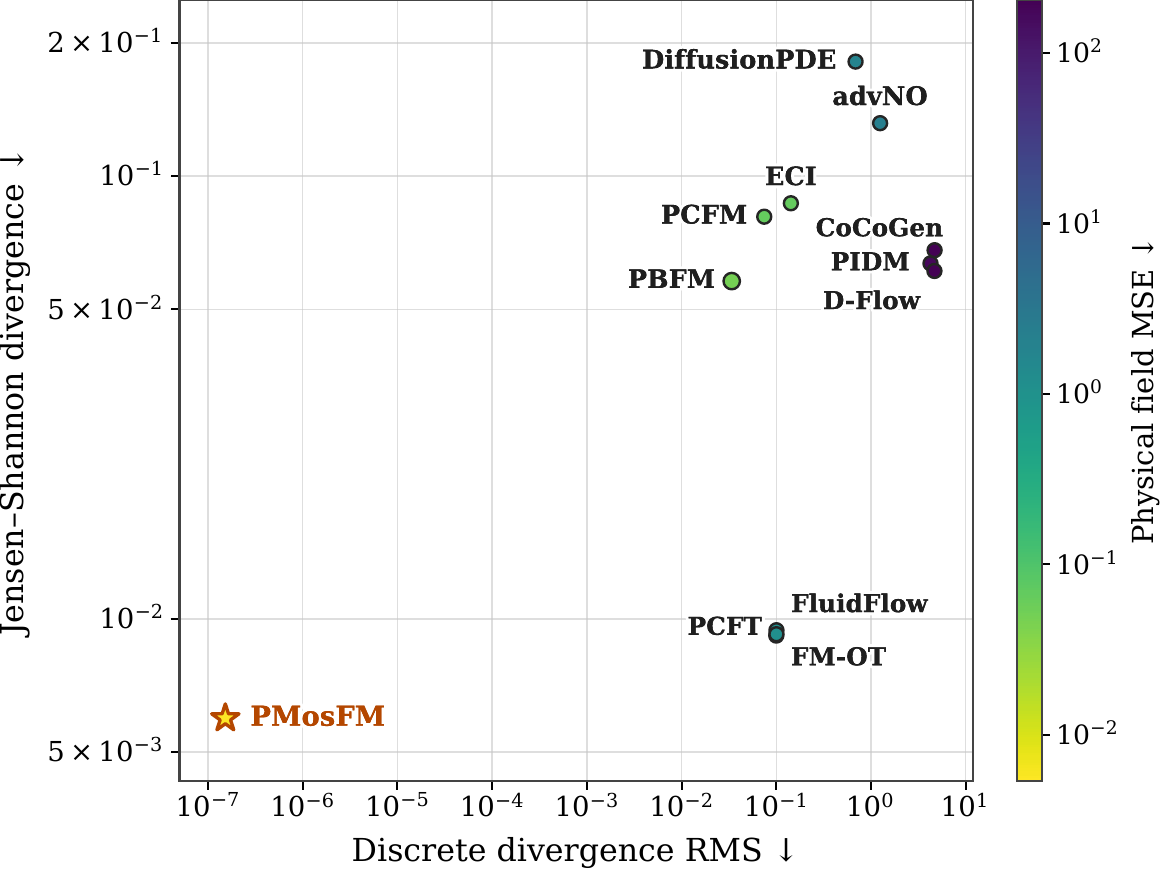}
\caption{Physical and sampling quality comparison.
Lower divergence and MSE are better.}
\label{fig:sampling-quality}
\end{wrapfigure}

\section{Experiments}
\label{sec:experiments}

\subsection{Experimental setup}
\label{sec:experimental_setup}

\paragraph{Datasets.}
We evaluate PMosFM on Darcy flow, dynamic stall, and Kolmogorov-flow generation and reconstruction~\citep{baldan2025pbfm}, together with Burgers~\citep{utkarsh2025pcfm}.
Turbulence forecasting and reconstruction use the existing turbulence
benchmark~\citep{oommen2026turbulence}.
We also evaluate a constrained-PDE reference suite with linear and nonlinear constraints.
Dataset and evaluator details are provided in Appendices~\ref{app:benchmark-taxonomy} and~\ref{app:pde-dataset-details}.

\paragraph{Baselines.}
We compare PMosFM with representative baselines that enforce physics constraints during training or sampling.
Training-time methods include PBFM~\citep{baldan2025pbfm};
sampling-time methods include ECI~\citep{Cheng2025}, PCFM~\citep{utkarsh2025pcfm},
D-Flow~\citep{ben2024d}, CoCoGen~\citep{Jacobsen2025}, and
DiffusionPDE~\citep{huang2024diffusionpde}.
Additional comparisons include PIDM~\citep{Bastek2025} and FM-OT~\citep{lipman2022flow}.
Results taken from prior work are distinguished from locally obtained results.

\begin{wraptable}{r}{0.62\textwidth}
\centering
\caption{Optimizer-update time (s) and peak CUDA memory (GB) across unrolling configurations.}
\label{tab:wall-clock-training}
\scriptsize
\setlength{\tabcolsep}{0.8pt}
\renewcommand{\arraystretch}{0.92}
\resizebox{\linewidth}{!}{%
\begin{tabular}{llcccccc}
\toprule
\multirow[c]{2}{*}{\raisebox{-0.35ex}{\textbf{Benchmark}}} & &
\multicolumn{5}{c}{\textbf{PBFM}} &
\multirow{2}{*}{\shortstack{\textbf{PMosFM}$^\dagger$\\ \textbf{0 unrolling}}} \\
\cmidrule(lr){3-7}
& & \textbf{1 step} & \textbf{2 steps} & \textbf{3 steps} & \textbf{4 steps} & \textbf{Sum} & \\
\midrule
\multirow[c]{2}{*}{Darcy} & [s] & 0.067 & 0.098 & 0.128 & 0.159 & 0.452 & \textbf{0.019} \\
& [GB] & 12.6 & 23.3 & 33.7 & 44.2 & 114 & \textbf{0.128} \\[0.7mm]
\multirow[c]{2}{*}{Kolmogorov} & [s] & 0.194 & 0.302 & 0.410 & 0.518 & 1.42 & \textbf{0.021} \\
& [GB] & 4.52 & 6.91 & 9.56 & 12.2 & 33.2 & \textbf{0.143} \\[0.7mm]
\multirow[c]{2}{*}{\shortstack[l]{Dynamic \\Stall}} & [s] & 0.081 & 0.118 & 0.155 & 0.190 & 0.544 & \textbf{0.056} \\
& [GB] & 4.41 & 7.18 & 9.90 & 12.6 & 34.1 & \textbf{0.258} \\[0.7mm]
\multirow[c]{2}{*}{\shortstack[l]{Turbulence \\reconstruction}} & [s] & 1.38 & 1.51 & 1.47 & 1.63 & 5.99 & \textbf{0.930} \\
& [GB] & 13.3 & 19.1 & 24.8 & 30.6 & 87.8 & \textbf{3.50} \\[0.7mm]
\multirow[c]{2}{*}{\shortstack[l]{Turbulence \\forecasting}} & [s] & 0.038 & 0.080 & 0.098 & 0.099 & 0.315 & \textbf{0.021} \\
& [GB] & 1.53 & 2.30 & 3.07 & 3.83 & 10.7 & \textbf{0.584} \\
\bottomrule
\end{tabular}%
}
\end{wraptable}

\subsection{Training and sampling efficiency}
\label{sec:exp-efficiency}

\paragraph{Training cost accounting.}
For each task, a common physical--distributional validation criterion defines the target quality.
Training cost includes encoding and preconditioner calibration, prerequisite training, and optimizer updates up to the first checkpoint meeting this criterion; validation overhead is reported separately unless full elapsed time is measured.
Table~\ref{tab:wall-clock-training} shows lower optimizer-update time and peak memory
for PMosFM than for every evaluated PBFM unrolling depth across all five tasks.
For PBFM, both quantities generally increase with unrolling depth \citep{baldan2025pbfm}.
PMosFM removes residual unrolling, while the endpoint branches and physical decoder remain part of the training cost.
The training trajectories in Fig.~\ref{fig:pmosfm-training-convergence} show earlier residual reduction for PMosFM, consistent with the convergence benefit sought by preconditioned manifold matching.
The optimization effects of preconditioning are evaluated in Section~\ref{sec:ablation}, with training efficiency detailed in Appendix~\ref{app:training-efficiency-cost}.

\begin{figure*}[t]
 \centering
 \includegraphics[width=\textwidth]{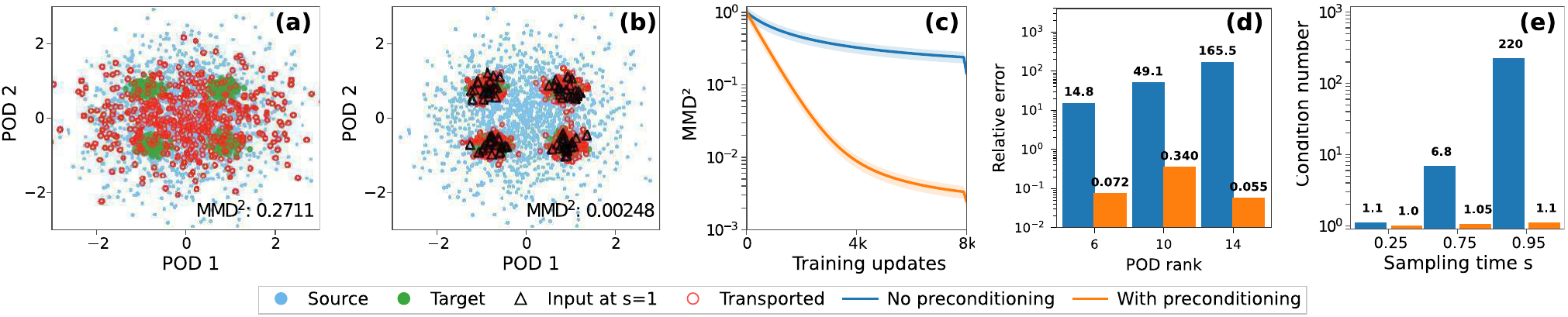}
 \caption{Preconditioning of feasible transport. (a) No preconditioning map. (b) With preconditioning map. (c) MMD$^2$ convergence. (d) Low-variance mode error. (e) Interpolation conditioning.}
 \label{fig:pmosfm-preconditioning}
\end{figure*}

\paragraph{Sampling quality and inference cost.}
We next compare sampling quality and inference time under each benchmark's evaluator.
Fig.~\ref{fig:sampling-quality} shows physical consistency, distributional accuracy, and field error across the evaluated samplers.
Further, Table~\ref{tab:sampling-quality-cost} compares physical and distributional errors, learned-network evaluations, and inference time across five benchmarks.
Baselines rely on multi-step learned transport, requiring 20--200 network evaluations, whereas PMosFM performs one-step evaluation followed by physical decoding.
In particular, PCFM enforces physical constraints during sampling and has higher inference time
than PMosFM in Table~\ref{tab:sampling-quality-cost}.
PMosFM also achieves competitive physical and distributional accuracy across
the benchmarks in Table~\ref{tab:sampling-quality-cost}.
Thus, PMosFM reduces sampling cost while maintaining competitive physical and distributional accuracy.
The sampling time includes physical decoding; decoder tolerances and costs are detailed in Appendix~\ref{app:chart-tolerance}.

\begin{figure*}[t]
\centering
\includegraphics[width=\textwidth]{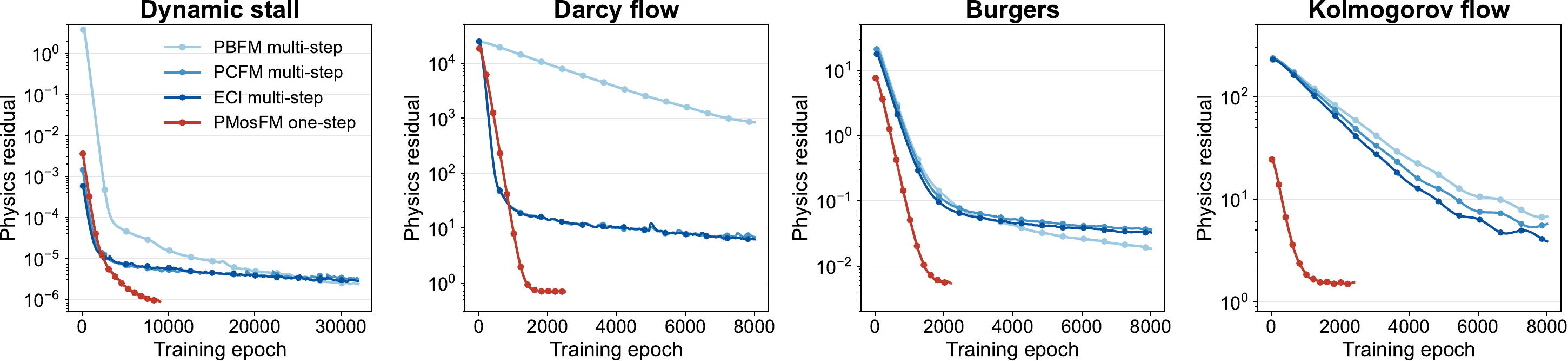}
\caption{Training-residual convergence across physics-constrained benchmarks.}
\label{fig:pmosfm-training-convergence}
\end{figure*}

\subsection{Physics-constrained generation}
\label{sec:exp-generation}

\noindent\textbf{Constraint satisfaction and field fidelity.}
Across the physics-constrained benchmarks, PMosFM preserves the principal solution
structure while keeping the encoded conservation residual near numerical
precision, whereas several baselines retain localized violations around dominant
transition regions
(see Fig.~\ref{fig:constrained-generation}).
The quantitative results further show competitive distributional accuracy
alongside constraint satisfaction
(see Table~\ref{tab:zero-shot-constrained-pde} in
Appendix~\ref{app:zero-shot-constrained-pde}),
indicating that improved physical consistency is achieved while retaining
competitive distributional accuracy.

\begin{figure*}[t]
    \centering
    \includegraphics[width=\linewidth]{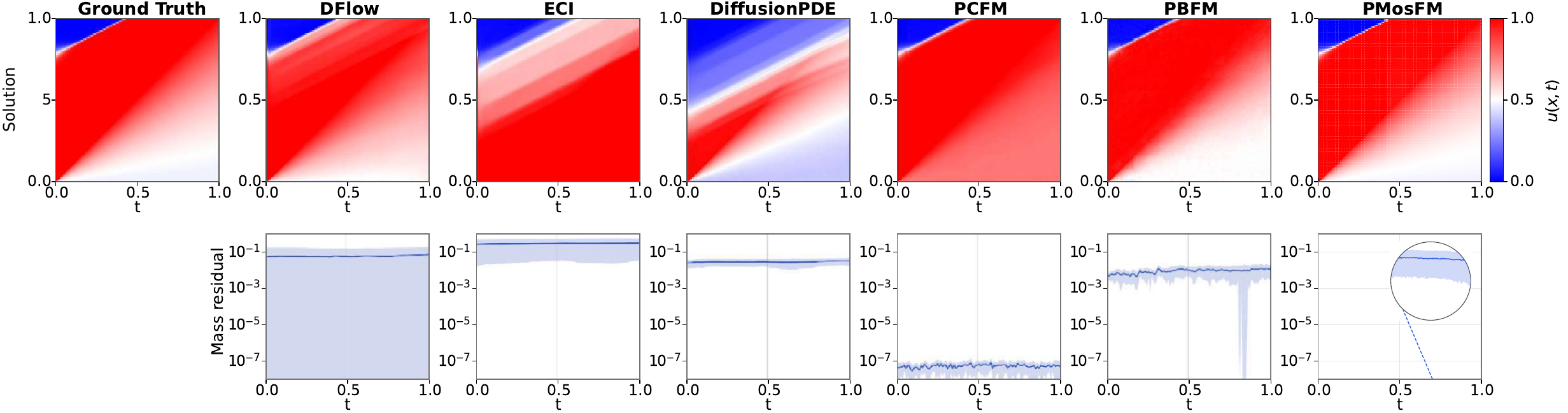}
    \caption{Generated fields and conservation residuals across methods.
    The top row compares generated solutions, while the bottom row shows the conservation-residual trajectories.
    }
    \label{fig:constrained-generation}
\end{figure*}

\noindent\textbf{Held-out turbulence diagnostics.}
On held-out turbulence, PMosFM reproduces the characteristic tear-drop topology
and non-Gaussian tails of the velocity-gradient statistics, while the baselines
underrepresent extreme-strain events
(see Fig.~\ref{fig:hit-forecasting-comparison} in
Appendix~\ref{app:heldout-turbulence-one-step}).
These statistics probe small-scale intermittency beyond the encoded constraint residual,
showing that PMosFM retains fine-scale turbulence structure in one-step generation.

\begin{table}[H]
\centering
\caption{Sampling quality and inference efficiency across benchmarks.
RE: physical residual error; WD: Wasserstein distance; JS: Jensen--Shannon divergence;
NFE: learned network evaluations; IT: inference time.}
\label{tab:sampling-quality-cost}
\small
\setlength{\tabcolsep}{3.2pt}
\renewcommand{\arraystretch}{0.8}
\resizebox{\linewidth}{!}{%
\begin{tabular}{@{}llccccccccc@{}}
\toprule
\textbf{Benchmark} & \textbf{Metric} & \textbf{FM-OT} & \textbf{CoCoGen} &
\textbf{PIDM} & \textbf{DiffusionPDE} & \textbf{D-Flow} & \textbf{ECI} &
\textbf{PCFM} & \textbf{PBFM} & \textbf{PMosFM} \\
\midrule
\multirow{5}{*}{Dynamic Stall}
& RE $\times 10^{6}$ $\downarrow$ & 11.0 & 376 & $2.96\times10^{3}$ & 12.2 & 11.3 & 40.7 & 0.143 & 0.339 & \textbf{0.0421} \\
& WD $\times 10^{4}$ $\downarrow$ & 2.71 & 132 & 179 & 2.51 & 3.48 & 11.1 & 4.01 & 1.81 & \textbf{0.842} \\
& JS $\times 10^{2}$ $\downarrow$ & 0.983 & 1.69 & 4.14 & 1.03 & 1.01 & 0.0570 & 7.46 & 0.680 & \textbf{0.0368} \\
& NFE $\downarrow$ & 20 & 100 & 100 & 20 & 20 & 200 & 20 & 20 & \textbf{1} \\
& IT [s] $\downarrow$ & 0.0598 & 0.184 & 0.0866 & 0.172 & 0.139 & 0.432 & 3.91 & 0.0605 & \textbf{0.00338} \\
\midrule
\multirow{5}{*}{Darcy Flow}
& RE $\downarrow$ & 4.16 & 1.32 & 0.0220 & 3.39 & 2.29 & 3.05 & $4.19\times10^{3}$ & 0.838 & \textbf{0.120} \\
& WD $\times 10^{2}$ $\downarrow$ & 0.0590 & 0.249 & 3.10 & 0.0890 & 0.147 & 2.89 & 3.74 & 0.138 & \textbf{0.0490} \\
& JS $\times 10^{1}$ $\downarrow$ & 0.131 & 0.360 & 3.18 & 0.139 & 0.237 & 2.82 & 0.199 & 0.256 & \textbf{0.0139} \\
& NFE $\downarrow$ & 20 & 100 & 100 & 20 & 20 & 20 & 20 & 20 & \textbf{1} \\
& IT [s] $\downarrow$ & 0.100 & 7.40 & 2.05 & 0.590 & 3.13 & 0.122 & 1.33 & 0.101 & \textbf{0.0178} \\
\midrule
\multirow{5}{*}{Burgers}
& RE $\times 10^{1}$ $\downarrow$ & 4.30 & 404 & $1.57\times10^{3}$ & 3.66 & $3.62\times10^{3}$ & 0.818 & 0.473 & 0.307 & \textbf{0.0435} \\
& WD $\times 10^{2}$ $\downarrow$ & 4.85 & 36.0 & 248 & 0.685 & 638 & 2.55 & 2.73 & 4.02 & \textbf{0.129} \\
& JS $\times 10^{2}$ $\downarrow$ & 1.92 & 12.3 & 31.5 & 0.184 & 52.4 & 0.809 & 0.884 & 1.43 & \textbf{0.0320} \\
& NFE $\downarrow$ & 20 & 20 & 20 & 20 & 20 & 20 & 20 & 20 & \textbf{1} \\
& IT [s] $\downarrow$ & 0.0349 & 0.0354 & 0.0354 & 0.227 & 0.0354 & 0.0902 & 0.134 & 0.0913 & \textbf{0.00237} \\
\midrule
\multirow{5}{*}{Kolmogorov Flow}
& RE $\times 10^{1}$ $\downarrow$ & 2.31 & 13.7 & 14.3 & 1.93 & 6.74 & 0.000 & 1.53 & 1.36 & \textbf{0.0842} \\
& WD $\times 10^{1}$ $\downarrow$ & 2.12 & 3.11 & 2.86 & 3.70 & 1.01 & 0.685 & 0.665 & 1.22 & \textbf{0.382} \\
& JS $\times 10^{2}$ $\downarrow$ & 12.5 & 29.1 & 29.2 & 19.4 & 16.8 & 7.10 & 6.87 & 7.44 & \textbf{1.84} \\
& NFE $\downarrow$ & 20 & 100 & 100 & 20 & 20 & 200 & 200 & 20 & \textbf{1} \\
& IT [s] $\downarrow$ & 0.0988 & 0.0449 & 0.0506 & 0.268 & 6.43 & 0.385 & 0.792 & 0.0990 & \textbf{0.00375} \\
\midrule
\multirow{5}{*}{\shortstack[l]{Turbulence Flow\\Reconstruction}}
& RE $\times 10^{1}$ $\downarrow$ & 5.28 & 1.52 & 12.1 & 0.767 & 9.52 & 0.521 & 0.521 & 0.192 & \textbf{0.0258} \\
& WD $\times 10^{2}$ $\downarrow$ & 2.47 & 1.35 & 3.68 & 1.01 & 2.85 & 1.54 & 1.49 & 1.19 & \textbf{0.684} \\
& JS $\times 10^{2}$ $\downarrow$ & 3.82 & 1.90 & 5.22 & 1.25 & 4.13 & 2.11 & 1.97 & 1.42 & \textbf{0.362} \\
& NFE $\downarrow$ & 20 & 20 & 20 & 20 & 20 & 20 & 20 & 20 & \textbf{1} \\
& IT [s] $\downarrow$ & 0.239 & 0.239 & 0.239 & 0.222 & 0.239 & 0.0301 & 0.0377 & 0.0268 & \textbf{0.00370} \\
\bottomrule
\end{tabular}%
}
\end{table}

\subsection{Underdetermined sparse reconstruction}
\label{sec:exp-underdetermined}

Physical constraints restrict the feasible set but leave sparse reconstruction underdetermined.
Fig.~\ref{fig:hit-mask-comparison} compares PMosFM and the baselines under identical observations and masks on a held-out state.
As missingness increases, the baselines lose spatial organization; with 99\% missing observations, they fail to recover the dominant structure of the reference kinetic-energy field, whereas PMosFM retains an identifiable large-scale pattern.
At 99\% missingness, PMosFM better preserves low-wavenumber energy and
velocity-gradient PDF tails than the baselines.
This test evaluates whether generated fields retain spatial and multiscale structure in unobserved regions beyond constraint satisfaction.
\begin{figure}[H]
\centering
\includegraphics[width=\textwidth]{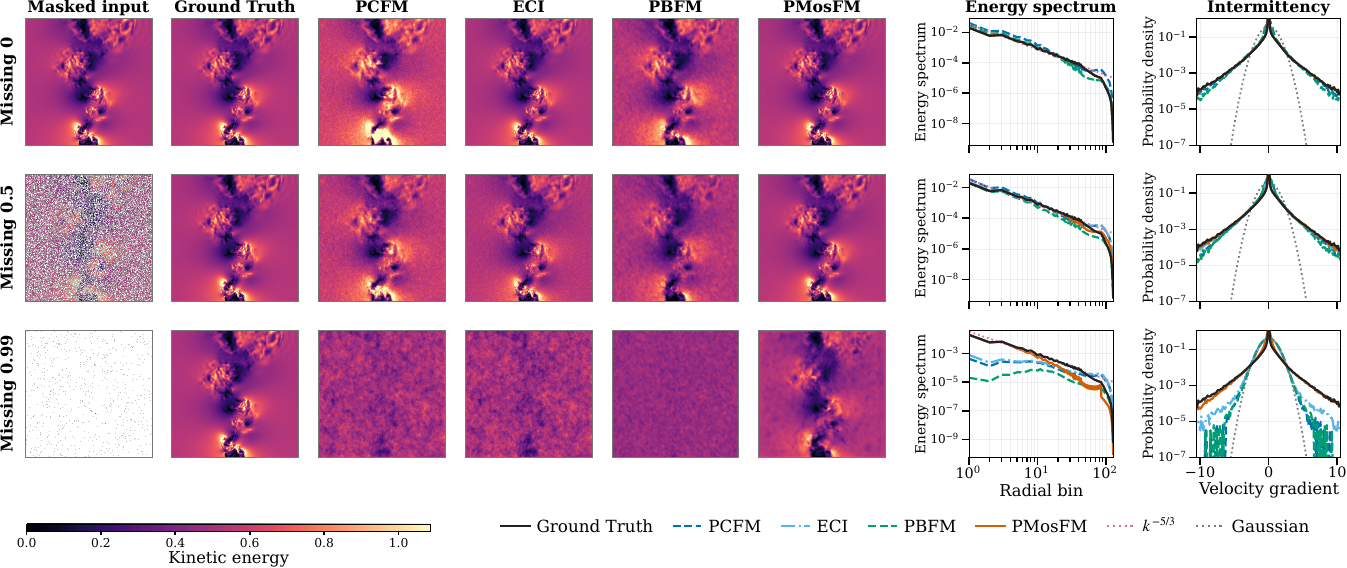}
\caption{Structural recovery under sparse observations.
PMosFM retains large-scale structure under masking across field, spectral, and velocity-gradient diagnostics.
}
\label{fig:hit-mask-comparison}
\end{figure}

\subsection{Ablation study}
\label{sec:ablation}

Ablation experiments isolate the effects of manifold parameterization, metric preconditioning, and endpoint regression.
Fig.~\ref{fig:pmosfm-preconditioning} shows that preconditioning improves distribution matching, reduces errors in low-variance modes, and stabilizes the interpolation covariance seen by the network.
Together, these diagnostics connect the improved training behavior to the statistical conditioning targeted by PMosFM.
Table~\ref{tab:pmosfm-one-step-ablation} benchmarks different one-step velocity regression formulations within PMosFM using energy distance (ED).
PMosFM gives the lowest ED on most benchmarks, whereas Shortcut is stronger on turbulence reconstruction.
The variation suggests that the preferred one-step objective depends on the physical task.
Additional ablation results are reported in Appendix~\ref{app:required-full-objective-ablation}.

\begin{table*}[t]
\centering
\caption{One-step-core ablation within PMosFM. Energy distance (ED; $\downarrow$) is reported; lower is better. KF denotes Kolmogorov flow. Bold denotes the lowest available value in each benchmark row.}
\label{tab:pmosfm-one-step-ablation}
\small
\setlength{\tabcolsep}{3.5pt}
\renewcommand{\arraystretch}{0.84}
\resizebox{\textwidth}{!}{%
\begin{tabular}{lccccccc}
\toprule
\multirow[c]{2}{*}{Benchmark}
& PMosFM- & PMosFM- & PMosFM- & PMosFM- & PMosFM- & PMosFM- & \textbf{PMosFM-} \\
& MeanFlow & iMF & W-Flow & Drifting & SoFlow & Shortcut & \textbf{native} \\
\midrule
Burgers
& 0.0758 & 0.0765 & 0.0518 & 0.0713 & 0.0714 & 0.0849 & \textbf{0.0374} \\
Dynamic Stall
& 0.1403 & 0.1315 & 0.0594 & 0.0442 & 0.1352 & 0.0455 & \textbf{0.0036} \\
Darcy Flow
& 0.0277 & 0.0341 & 0.0163 & 0.0233 & 0.0317 & 0.0158 & \textbf{0.0043} \\
KF Reconstruction
& 0.0144 & 0.0145 & 0.0574 & 0.0476 & 0.0249 & 0.0104 & \textbf{0.0037} \\
KF Generation
& 2.0645 & 2.1278 & 1.6020 & 2.1451 & 2.0887 & 1.9672 & \textbf{1.5922} \\
Turbulence Reconstruction
& 1.0776 & 1.0836 & 1.0829 & 1.0829 & 1.0727 & \textbf{0.9423} & 0.9632 \\
Turbulence Forecasting
& 1.5448 & 1.4719 & 0.8767 & 0.8658 & 1.5019 & 0.9752 & \textbf{0.7507} \\
\bottomrule
\end{tabular}%
}
\end{table*}

\section{Conclusion}
\label{sec:conclusion}

In this work, we present PMosFM, a precondition--manifold--matching framework for one-step physics-constrained generation that separates encoded feasibility from learned transport.
Feasible coordinates remove encoded residual-normal directions, while geometric and interpolation-state preconditioning condition the remaining transport, and finite-interval matching compares decoded endpoints in physical space.
Across the evaluated PDE settings, PMosFM achieves competitive physical and distributional fidelity with near-zero encoded residuals; controlled studies show better conditioning and faster distributional convergence with preconditioning.
PMosFM also reduces measured optimizer-update time relative to PBFM, and one-step sampling with physical decoding lowers inference time compared with the evaluated multi-step baselines.

\section*{AI Use Statement}
Generative AI tools were used to assist with manuscript organization,
language editing, revision, and figure preparation.
All AI-assisted content was reviewed and revised by the authors;
experimental measurements were produced by the reported computational
pipelines, and the mathematical derivations, experimental results,
references were independently verified.
The authors take full responsibility for the final content of this work,
including all AI-assisted text and artifacts.

\bibliography{iclr2027_conference}
\bibliographystyle{iclr2027_conference}

\newpage
\appendix
\onecolumn

\section{Notation}
\label{app:notation}

This table summarizes the notation adopted throughout the paper and appendix.
Symbols are grouped by residual manifold, preconditioning, and finite-interval matching.

\begin{table}[h]
\small
\centering
\renewcommand{\arraystretch}{1.4}
\caption{Summary of the notation adopted throughout the paper and appendix.}
\label{tab:notation}
\setlength{\tabcolsep}{3pt}
\begin{tabular}{|>{\centering\arraybackslash}m{0.17\textwidth}|>{\raggedright\arraybackslash}m{0.19\textwidth}|>{\raggedright\arraybackslash}m{0.54\textwidth}|}
\hline
\rowcolor{gray!30}
\textbf{Symbol} & \textbf{Name} & \textbf{Description} \\
\hline
$(x,\mathcal X_h,c,h)$ & state space, condition, mesh & $x\in\mathcal X_h\subseteq\mathbb R^d$ is a discretized physical state; $c$ is the condition; $h$ is the spatial discretization. \\
$R_h(x,c)\in\mathbb R^{n_R}$ & discrete residual & Residual that evaluates the hard physical constraint; $n_R$ is the number of residual components. \\
$(\mathcal Z_c,\mathcal M_c^\chi)$ & residual-zero set & $\mathcal Z_c=\{x:R_h(x,c)=0\}$; $\mathcal M_c^\chi=\chi_c(\mathcal Y_c)\subseteq\mathcal Z_c$ is the component represented by the model. \\
$(y,\mathcal Y_c,m)$ & intrinsic coordinates & $y\in\mathcal Y_c\subset\mathbb R^m$ represents the independent degrees of freedom; $m=d-q$ under the regular-rank assumption. \\
$(\chi_c,\eta_c)$ & constraint-satisfying parameterization and encoder & $R_h\circ\chi_c\equiv0$; $\eta_c$ encodes feasible data or provides a measurable selection. \\
$(J_\chi,J_R)$ & coordinate and residual Jacobians & $J_\chi=D_y\chi_c$, $J_R=D_xR_h$ satisfy $J_RJ_\chi=0$; regular-rank conditions identify $\mathop{\mathrm{range}}\nolimits J_\chi$ with the tangent space. \\
$(M_h,M_x)$ & physical metric & Fixed positive-definite physical or quadrature metric; $M_x$ abbreviates $M_h$ when $h$ is implicit. \\
$G_c(y)$ & pullback metric & Physical metric in intrinsic coordinates: $G_c(y)=J_\chi(y,c)^\top M_hJ_\chi(y,c)$. \\
\hline
\rowcolor{gray!15}
$(C_c,\mu_c,\Psi_c)$ & geometric setting & $r=C_c(y-\mu_c)$ and $\Psi_c(r)=\chi_c(\mu_c+C_c^{-1}r)$. \\
$(r_0,r_1,r_s,w)$ & interpolation pair & $r_s=(1-s)r_0+sr_1$ and $w=r_1-r_0$. \\
$(m_{s,c},\Sigma_{s,c},P_{s,c})$ & input preconditioner & Mean, covariance, and $P_{s,c}=(\Sigma_{s,c}+\varepsilon_P I)^{-1/2}$. \\
$(T_\theta^{s,t},u_\theta)$ & finite-interval map and velocity & $T_\theta^{s,t}(r,c)=r+(t-s)u_\theta(P_{s,c}(r-m_{s,c}),s,t,c)$. \\
$(a_\theta,b_{\bar\theta})$ & endpoint estimates & Endpoint estimates from nearby interpolation times, with gradients stopped through the second branch. \\
$(\mathcal L_{\mathrm{FM}},\mathcal L_{\mathrm{PE}})$ & matching losses & Diagonal velocity FM loss and off-diagonal decoded-endpoint loss. \\
\hline
\rowcolor{gray!15}
$(e_\theta,d_\theta)$ & map defects & Coordinate flow-map defect and the corresponding decoded physical defect. \\
$(\gamma,\delta,\varepsilon_P)$ & matching parameters & Branch weight, finite-difference interval, and covariance regularizer. \\
\hline
\rowcolor{gray!15}
$(a,J_G)$ & differentiation variable and endpoint Jacobian & $a$ is an input or parameter; $J_G=D_aG_\phi$ is the corresponding endpoint Jacobian. \\
$H_{\mathrm{GN},a}^{R}$ & residual Gauss--Newton block & $J_G^\top J_R^\top W_RJ_RJ_G$; exact residual factorization makes this block vanish for the encoded residual after composition with the decoder. \\
$(\varepsilon_{\mathrm{chart}},\delta_{\mathrm{chart}})$ & approximation error bounds & Bounds on residual magnitude and pulled-back residual differential. \\
$(\sigma_{\min}^{+},\kappa_{+})$ & active spectral quantities & Smallest nonzero singular value and condition number on the active subspace. \\
\hline
\end{tabular}
\end{table}

\section{Detailed Training and Sampling Procedures}
\label{app:pmosfm-algorithms}

The main-text algorithm exposes the computation that distinguishes PMosFM.
This appendix separates calibration and training from one-step physical sampling.

\begin{figure}[H]
\centering
\begin{minipage}[t]{0.57\linewidth}
\begin{algorithm}[H]
\caption{Full training procedure of PMosFM}
\label{alg:pmosfm-training-full}
\small
\begin{algorithmic}[1]
\Require Data $p_{\mathrm{data}}(x\mid c)$; source $p_0(r\mid c)$
\Statex Model $u_\theta$; encoder--decoder $(\eta_c,\chi_c)$
\Statex Metric $M_h$; $\gamma$; $\varepsilon_P$
\Ensure Trained $\theta$ and frozen $(\mu_c,C_c,m_{s,c},P_{s,c})$
\State Estimate $(\mu_c,C_c)$ from the training data
\State Using $r_1=C_c(\eta_c(x_1)-\mu_c)$ and
\Statex $r_s=(1-s)r_0+sr_1$, estimate $m_{s,c}\gets\mathbb E[r_s\mid c]$,
\Statex $\Sigma_{s,c}\gets\operatorname{Cov}(r_s\mid c)$, $\quad P_{s,c}\gets(\Sigma_{s,c}+\varepsilon_P I)^{-1/2}$
\State Fix $(\mu_c,C_c,m_{s,c},P_{s,c})$
\While{not converged}
    \State Select $c$; draw $0\le s<t\le1$ and $0<\delta<t-s$
    \State $x_1\sim p_{\mathrm{data}}(\cdot\mid c)$, $r_0\sim p_0(\cdot\mid c)$
    \Statex $\qquad r_1\gets C_c(\eta_c(x_1)-\mu_c)$
    \State $r_s\gets(1-s)r_0+sr_1$, $w\gets r_1-r_0$
    \Statex $\qquad r_{s+\delta}\gets r_s+\delta w$
    \State $\displaystyle
    \ell_{\mathrm{FM}}\gets
    \frac{1}{2m}
    \left\|
    u_\theta\!\left(P_{s,c}(r_s-m_{s,c}),s,s,c\right)-w
    \right\|_2^2$
    \State $\bar\theta\gets\operatorname{sg}(\theta)$
    \State $a_\theta\gets T_\theta^{s,t}(r_s,c)=r_s+(t-s)u_\theta\!\bigl($
    \Statex $\qquad P_{s,c}(r_s-m_{s,c}),s,t,c\bigr)$
    \State $b_{\bar\theta}\gets T_{\bar\theta}^{s+\delta,t}(r_{s+\delta},c)=r_{s+\delta}+(t-s-\delta)u_{\bar\theta}\!\bigl($
    \Statex $\qquad P_{s+\delta,c}(r_{s+\delta}-m_{s+\delta,c}), s+\delta,t,c\bigr)$
    \State $\displaystyle
    \ell_{\mathrm{PE}}\gets
    \frac{
    \left\|
    \Psi_c(a_\theta)
    -\operatorname{sg}\!\left[\Psi_c(b_{\bar\theta})\right]
    \right\|_{M_h}^{2}}
    {2m\,\delta(t-s)}$
    \State $\displaystyle
    \theta\gets
    \operatorname{Update}\!\left[
    \theta,\,
    \gamma\ell_{\mathrm{FM}}
    +(1-\gamma)\ell_{\mathrm{PE}}
    \right]$
\EndWhile
\State \Return $\theta$ and $(\mu_c,C_c,m_{s,c},P_{s,c})$
\end{algorithmic}
\end{algorithm}
\end{minipage}
\hfill
\begin{minipage}[t]{0.39\linewidth}
\begin{algorithm}[H]
\caption{Full one-step sampling procedure of PMosFM}
\label{alg:pmosfm-sampling-full}
\small
\begin{algorithmic}[1]
\Require Trained $\theta$; source $p_0(r\mid c)$; condition $c$;
frozen $(\mu_c,C_c,m_{0,c},P_{0,c})$;
decoder $\Psi_c$
\Ensure Physics-consistent sample $\widehat x_1$
\State Sample $r_0\sim p_0(\cdot\mid c)$
\State Form the centered source coordinate
\[
r_0-m_{0,c}
\]
\State Apply the frozen input preconditioner
\[
P_{0,c}(r_0-m_{0,c})
\]
\State Evaluate the full-interval map once
\[
\widehat r_1
\gets
T_\theta^{0,1}(r_0,c)
\]
\State Expand the one-step map as
\[
\widehat r_1
=
r_0+
u_\theta\!\left(
P_{0,c}(r_0-m_{0,c}),0,1,c
\right)
\]
\State Apply the inverse geometric transform
\[
\mu_c+C_c^{-1}\widehat r_1
\]
\State Decode on the feasible manifold
\[
\widehat x_1
\gets
\chi_c\!\left(
\mu_c+C_c^{-1}\widehat r_1
\right)
=
\Psi_c(\widehat r_1)
\]
\State \Return $\widehat x_1$
\end{algorithmic}
\end{algorithm}
\end{minipage}
\label{fig:pmosfm-detailed-procedures}
\end{figure}

The calibration quantities are estimated from the training split and then held fixed.
Algorithm~\ref{alg:pmosfm-training-full} expands the two-time objective in Eq.~\ref{eq:pmosfm-loss}, while Algorithm~\ref{alg:pmosfm-sampling-full} evaluates $T_\theta^{0,1}$ once before physical decoding.
Iterations in a solver-based decoder belong to physical decoding rather than learned transport and are accounted for separately in Appendix~\ref{app:chart-tolerance}.

\section{Dataset Details}
\label{app:pde-dataset-details}

This section records the PDE system, data split, conditioning variables, normalization, and physical diagnostic for each benchmark.
The evaluated tasks cover elliptic flow, nonlinear conservation laws, compressible aerodynamics, turbulence, and inverse problems.
All normalization statistics are estimated on the training partition and remain fixed for validation and test data.
Generated fields are mapped back to physical units before residual evaluation whenever inverse normalization is available.
Each benchmark-specific comparison shares one split and one evaluator.
When the benchmark provides the state, forcing, manifold, and boundary variables required by the governing equations, we evaluate the full prescribed residual; otherwise, we report the identifiable residual components.

\subsection{Darcy Flow}
Each sample contains a pressure field \(p(\vx)\) and permeability field \(K(\vx)\) on a two-dimensional grid.
The data follow the steady Darcy system
\begin{equation*}\vu=-K\nabla p,
\qquad
\nabla\!\cdot\vu=f,
\qquad
\vu\!\cdot\vn=0\ \text{on }\partial\Omega,
\qquad
\int_\Omega p\,\mathrm{d}\vx=0.
\end{equation*}
The benchmark pairs log-Gaussian permeability with steady Darcy pressure, following PBFM~\citep{baldan2025pbfm}.
Both fields are standardized by channel.
For common evaluation, predictions are de-normalized before applying the prescribed residual operator.

\subsection{Dynamic Stall}
The data are obtained from unsteady compressible two-dimensional RANS simulations of a sinusoidally pitching NACA0012 airfoil~\citep{baldan2025pbfm}.
The operating condition contains the free-stream Mach number, mean angle of attack, pitching amplitude, and reduced frequency.
The dataset contains 128 nominal training conditions with 32 realizations per condition and 16 held-out conditions.
The six generated channels encode pressure, two tangential-velocity-gradient components, temperature, density, and signed wall shear stress on a \(128\times128\) spatio-temporal surface grid; two additional channels provide the local surface direction.
After de-normalization, the two algebraic residuals are
\begin{equation*}r_{\mathrm{gas}}=p-\rho R T,
\qquad
r_{\tau}=\tau_w-\mathop{\mathrm{sign}}\nolimits(\vmu\!\cdot\vs)
\lVert\vmu\rVert_2,
\end{equation*}
where \(\vmu=\mu(T)(\partial_x u_s,\partial_y u_s)\), \(\vs\) is the local surface direction, and \(\mu(T)\) is given by Sutherland's law.
The manifold-conditioned implementation follows the same local direction as the external residual, so the parameterization and evaluator share a sign convention.

\subsection{Burgers}
The Burgers benchmark consists of trajectories at viscosity \(\nu=10^{-3}\).
The first 1,000 trajectories form the training partition, indices 1,000--1,499 form validation, and the remaining trajectories form the test partition.
We remove the duplicated terminal slice and standardize the resulting space--time field with training statistics.
For a generated trajectory \(u(t,x)\), the diagnostic is
\begin{equation*}r_{\mathrm{B}}=\partial_tu+u\,\partial_xu-\nu\,\partial_{xx}u,
\end{equation*}
with centered differences, periodic padding in space, and replicated temporal endpoints.
We also apply periodic RK4 pseudospectral evolution; because its discretization differs from the finite-difference evaluator, we report solver results separately from residual RMS.

\subsection{Turbulence Mass Transfer}
Each sample contains horizontal velocity, vertical velocity, pressure, and signed distance to the immersed manifold as four aligned channels.
The first 900 samples form the training set, and the remainder forms the validation set with channel-wise training normalization.
From these variables, we define a field-level manifold-consistency evaluator comprising discrete incompressibility, normal velocity near the zero level set, signed-distance Eikonal error, and weak pressure smoothness:
\begin{equation*}R_{\mathrm{TMT}}=
\bigl[\nabla\!\cdot\vu,
e^{-(\phi/\delta)^2}\vu\!\cdot\vn_\phi,
0.2(\lVert\nabla\phi\rVert_2-1),
0.02\Delta p\bigr].
\end{equation*}

\subsection{Kolmogorov Flow}
The conditional dataset is generated at Reynolds numbers in \([100,500]\) on a \(128\times128\) periodic grid.
The training set covers 32 Reynolds numbers and the held-out set 16 Reynolds numbers, with 1,024 statistically stationary velocity snapshots per condition.
The two output channels are the planar velocity components, and the Reynolds number is the conditioning variable.
Both the encoded residual and the independent parameterization comparison measure incompressibility,
\begin{equation*}r_{\mathrm{KF}}=\partial_xu+\partial_yv,
\end{equation*}
after restoring channel scales.
The streamfunction parameterization and Fourier evaluator use compatible periodic derivatives.
Central differences reveal grid sensitivity in the residual.

\subsection{Turbulence Forecasting}
For turbulence forecasting, we adopt the homogeneous-isotropic-turbulence benchmark of \citet{oommen2026turbulence}.
Each state contains four physical channels on a three-dimensional periodic grid.
Each target is conditioned on the preceding four frames, with a one-frame forecast horizon.
After the four-frame source offset, frames 0--159 form the training set, frames 160--179 form the validation set, and the remaining complete windows form the test set.
Normalization follows the benchmark means and standard deviations when available.
The physical consistency diagnostic is the periodic three-dimensional divergence of the first three channels,
\begin{equation*}r_{\mathrm{div}}=\partial_xu+\partial_yv+\partial_zw,
\end{equation*}
A Fourier Helmholtz projection acts on the three velocity channels; the fourth is unchanged.

\subsection{Turbulence Flow Reconstruction}
This task contains four-channel, three-dimensional Gen4Turbulence fields and constructs random point masks on each frame.
Frames 0--149, 150--167, and 168 onward define the training, validation, and test partitions, respectively.
The condition concatenates the normalized observed field \(\vy=\mathbf{M}\odot\vx\) and the binary mask \(\mathbf{M}\); validation and test masks are fixed functions of the frame index.
The common residual stacks measurement consistency with velocity divergence,
\begin{equation*}R_{\mathrm{rec}}(\vx;\vy,\mathbf{M})=
\bigl[\tfrac14(\mathbf{M}\odot\vx-\vy),\ \nabla\!\cdot\vu\bigr].
\end{equation*}
The one-pass decoder first performs a Fourier divergence-free projection and then overwrites observed entries.
Because the observation overwrite can reintroduce divergence after the Fourier projection, we evaluate measurement consistency and incompressibility separately.

\subsection{Kolmogorov-flow Generation}
The KF generation benchmark consists of planar Kolmogorov-flow velocity fields \((u,v)\), represented as trajectories or individual snapshots.
Trajectory and time axes are flattened into a common sample axis, followed by a deterministic 90/10 train--validation partition and channel-wise training normalization.
The diagnostic concatenates periodic divergence, a weak Laplacian penalty on vorticity, and a weak Laplacian penalty on kinetic energy:
\begin{equation}R_{\mathrm{KFG}}=
\bigl[\nabla\!\cdot\vu,\ 0.02\Delta\omega,\ 0.01\Delta(u^2+v^2)\bigr].
\label{eq:app-kfg-residual}
\end{equation}
Divergence checks incompressibility; the other terms capture high-frequency artifacts.

\subsection{Kolmogorov-flow Reconstruction}
The target fields, split, and normalization are identical to Kolmogorov-flow Generation; the model is conditioned on a low-resolution observation.
During training, a factor is drawn from the set \(\{1,2,4\}\); the velocity is subsampled by that factor and bilinearly returned to the target grid before conditioning the model.
The evaluation residual is the same \(R_{\mathrm{KFG}}\) as in Eq.~\ref{eq:app-kfg-residual}, while reconstruction error is computed against the full-resolution held-out target.
Separating observation fidelity from flow diagnostics keeps reconstruction accuracy and physical consistency as distinct criteria.

\section{Additional Experiments}
\label{app:main-result-tables}

This appendix reports supplementary evaluation protocols and results for the
PMosFM objective in Eq.~\ref{eq:pmosfm-loss}.
Validation data are used for model and checkpoint selection, and all
normalization statistics are estimated from the training partition.
PMosFM preconditioners are calibrated on the training data and held fixed for
validation and test evaluation.

\paragraph{Measures and comparison protocol.}
Residual checks are applied in physical units whenever inverse normalization is available.
The analytic study of the feasible set and target distribution reports residual RMS, KL divergence, and total variation against the specified co-area law.
The conditioning study reports active condition numbers and the iterations required by a fixed stable-step rule.
The constrained-PDE reference protocol reports MMSE, SMSE, constraint error, and Fr\'echet Poseidon distance; lower values indicate closer agreement with the reference distribution or constraint.
Physics-constrained generation evaluations report task-specific field errors, prescribed residuals or named proxies, and distributional statistics.
Timing records separate learned-network evaluations from preconditioning and physical decoding costs.
Each table caption identifies the common evaluator and whether the comparison shares training budgets or matched quality.

\begin{table}[H]
\centering
\caption{Summary of constraint types in the experimental suite, categorized by mathematical form and scope.}
\label{tab:pmosfm-constraint-charts}
\small
\setlength{\tabcolsep}{4pt}
\renewcommand{\arraystretch}{1.13}
\begin{tabularx}{\linewidth}{@{}>{\raggedright\arraybackslash}p{0.24\linewidth}>{\centering\arraybackslash}p{0.49\linewidth}>{\raggedright\arraybackslash}X@{}}
\toprule
\textbf{Constraint type} & \textbf{Representative form} & \textbf{Linearity} \\
\midrule
Affine IC / BC / observations & \(\begin{gathered}R_h(x,c)=Ax-b=0,\\ x=x_p+Ny,\; AN=0\end{gathered}\) & Linear \\
Global conservation & \(\begin{gathered}\ell^{\top}x-C=0,\\ \text{conservation-preserving or mean-free coordinates}\end{gathered}\) & Linear \\
Differential compatibility & \(\begin{gathered}D_xu+D_yv=0,\\ (u,v)=(D_y\psi,-D_x\psi)\end{gathered}\) & Linear \\
Algebraic constitutive relation & \(\begin{gathered}g(x,c)=0,\\ \text{eliminate dependent variables}\end{gathered}\) & Nonlinear \\
Implicit PDE constraint & \(\begin{gathered}F_h(z^\star(y,c),y,c)=0,\\ z=z^\star(y,c)\end{gathered}\) & Potentially nonlinear \\
General nonlinear coupled constraint & \(\begin{gathered}R_h(x,c)=0,\\ \text{local parameterization or approximate decoder}\end{gathered}\) & Nonlinear \\
\bottomrule
\end{tabularx}
\end{table}

\subsection{Comprehensive Training Cost and Memory Footprint}
\label{app:training-efficiency-cost}

For a complete method, training and sampling costs are accounted for separately as
\begin{align*}T_{\mathrm{train}}(\epsilon)&=T_{\mathrm{calibration}}+N_{\mathrm{update}}(\epsilon)\,\overline t_{\mathrm{update}},
\\
T_{\mathrm{sample}}&=T_{\mathrm{network}}+T_{\mathrm{precondition}}+T_{\mathrm{decode}}.
\end{align*}
where $\epsilon$ denotes a common quality criterion, $N_{\mathrm{update}}(\epsilon)$ is the number of optimizer updates required to first meet this criterion, and $\overline t_{\mathrm{update}}$ is the mean update time.
$T_{\mathrm{calibration}}$ includes estimation of the fixed preconditioner, while $T_{\mathrm{network}}$, $T_{\mathrm{precondition}}$, and $T_{\mathrm{decode}}$ denote learned-network evaluation, preconditioning, and physical decoding costs, respectively.

For a curriculum of PBFM unrolling depths, the optimizer-update component is computed from the actual number of updates executed at each depth:
\begin{equation*}T_{\mathrm{updates}}=\sum_n U_n\,\overline t_n,
\end{equation*}
where $U_n$ counts executed updates and $\overline t_n$ is their mean time at depth $n$.
The single-update measurements in Table~\ref{tab:wall-clock-training} represent separate unrolling configurations, not additive components; prerequisite training and calibration count once, and validation overhead is included only in full elapsed time.

\paragraph{Implementation details.}
Training-cost profiling used an NVIDIA RTX PRO 6000 Blackwell Server Edition
(96~GB VRAM) with PyTorch 2.7.0+cu128 and batch size 8. We recorded 80
optimizer updates and discarded the first ten as warmup; CUDA timing was
synchronized, and peak memory was measured with
\texttt{torch.cuda.max\_memory\_allocated}.

Table~\ref{tab:full-aligned-training-cost} compares optimizer-update time and
peak memory across five physical benchmarks. PMosFM has the lowest peak memory
on four benchmarks and the second-lowest on turbulence reconstruction; its
update time is lowest on Burgers and second-lowest on turbulence
reconstruction. This pattern is consistent with separating feasibility from
learned transport: feasible coordinates and the removal of terminal residual
unrolling reduce memory and update cost relative to PBFM.

\begin{table*}[t]
\centering
\caption{Training time (s per optimizer update) and peak allocated GPU memory (GB) across the benchmarks and baselines in Table~\ref{tab:sampling-quality-cost}. Lower is better; bold marks PMosFM.}
\label{tab:full-aligned-training-cost}
\small
\renewcommand{\arraystretch}{1.25}
\resizebox{\linewidth}{!}{
\begin{tabular}{llccccccccc}
\toprule
\textbf{Benchmark} & \textbf{Metric} & \textbf{FM-OT} & \textbf{CoCoGen} & \textbf{PIDM} & \textbf{DiffusionPDE} & \textbf{D-Flow} & \textbf{ECI} & \textbf{PCFM} & \textbf{PBFM} & \textbf{PMosFM} \\
\midrule
\multirow{2}{*}{\textbf{Dynamic Stall}} 
 & Train s/iter $\downarrow$ & 0.0183 & 0.0090 & 0.0094 & 0.0511 & 0.0470 & 0.0169 & 0.0183 & 0.5440 & \textbf{0.0563} \\
 & Peak GB $\downarrow$      & 0.4027 & 0.9467 & 0.9465 & 2.1603 & 0.9475 & 0.4027 & 0.4027 & 34.1000 & \textbf{0.2580} \\
\midrule
\multirow{2}{*}{\textbf{Darcy Flow}} 
 & Train s/iter $\downarrow$ & 0.0289 & 0.0073 & 0.0065 & 0.0517 & 0.0170 & 0.0276 & 0.0289 & 0.4520 & \textbf{0.0185} \\
 & Peak GB $\downarrow$      & 0.5621 & 0.2412 & 0.2412 & 2.1569 & 0.2418 & 0.5621 & 0.5621 & 114.0000 & \textbf{0.1280} \\
\midrule
\multirow{2}{*}{\textbf{Burgers}} 
 & Train s/iter $\downarrow$ & 0.0246 & 0.0249 & 0.0245 & 0.0513 & 0.0240 & 0.0091 & 0.0091 & 0.0126 & \textbf{0.0057} \\
 & Peak GB $\downarrow$      & 2.2491 & 2.2465 & 2.2465 & 2.1565 & 2.2465 & 0.9487 & 0.9487 & 0.7245 & \textbf{0.5676} \\
\midrule
\multirow{2}{*}{\textbf{Kolmogorov Flow}} 
 & Train s/iter $\downarrow$ & 0.0438 & 0.0088 & 0.0092 & 0.0940 & 0.0146 & 0.0420 & 0.0438 & 1.4200 & \textbf{0.0208} \\
 & Peak GB $\downarrow$      & 1.5317 & 0.9359 & 0.9359 & 2.1569 & 0.9350 & 1.5317 & 1.5317 & 33.2000 & \textbf{0.1430} \\
\midrule
\multirow{2}{*}{\textbf{Turbulence Flow Rec.}} 
 & Train s/iter $\downarrow$ & 0.9988 & 2.0501 & 1.0805 & 0.0766 & 1.0249 & 1.0076 & 0.9988 & 5.9900 & \textbf{0.9300} \\
 & Peak GB $\downarrow$      & 9.7720 & 16.5738 & 16.5738 & 2.1596 & 16.5738 & 9.7720 & 9.7720 & 87.8000 & \textbf{3.5000} \\
\bottomrule
\end{tabular}
}
\begin{minipage}{0.98\linewidth}
\footnotesize
PBFM entries use the four-step Sum and PMosFM entries the one-step measurement.
\end{minipage}
\end{table*}

\subsection{Benchmark summary}
\label{app:benchmark-taxonomy}

Table~\ref{tab:pmosfm-constraint-charts} summarizes the constraint forms covered by the experimental suite.
The constrained-PDE reference suite combines affine initial or boundary constraints with linear or nonlinear conservation laws.
The physics-constrained generation studies additionally cover compatible flow fields, algebraic closures, solver-embedded PDE decoders, and coupled constraints.
A parameterization specifies feasible support, while each task separately declares the target law and physical evaluator.

\begin{table*}[t]
\centering
\caption{Benchmark summary.
Physical checks assess conservation; statistical metrics assess coverage.}
\label{tab:benchmark-matrix}
\tablefont
\setlength{\tabcolsep}{3.8pt}
\begin{tabularx}{\textwidth}{@{}l l Y Y@{}}
\toprule
Benchmark & Task & Physical check & Distributional check \\
\midrule
Darcy flow & Field generation & PDE/boundary residual & WD/JS; field moments \\
Dynamic stall & Conditional field & Wall-shear residual & MSE; shock statistics \\
Burgers & Nonlinear field & Conservation/boundary residual & WD/JS; shock location \\
\midrule
Kolmogorov flow generation & Turbulence field & Incompressibility/flow statistics & MMD/WD; spectra; diversity \\
Kolmogorov flow reconstruction & Conditional inverse & Observation/flow consistency & MSE; CRPS/coverage \\
\midrule
Turbulence forecast & Forecasting & Temporal/flow diagnostics & Forecast MSE; spectra; long-horizon stats \\
Turbulence reconstruction & Sparse inverse & Observation/flow consistency & MSE; CRPS/coverage; spectra \\
\bottomrule
\end{tabularx}
\end{table*}

\subsection{Zero-shot constrained-PDE reference}
\label{app:zero-shot-constrained-pde}

We compare sampling methods across PDEs with linear and nonlinear constraints under the common pretrained FFM backbone of the constrained-PDE protocol~\citep{utkarsh2025pcfm}.
Table~\ref{tab:benchmark-matrix} pairs each task's physical checks with
field or distributional metrics, separating constraint satisfaction from sample quality.
Following \citet{kerrigan2023functional} for pointwise functional statistics and \citet{Cheng2025} for constraint-aware feature-space evaluation, Table~\ref{tab:zero-shot-constrained-pde} reports pointwise mean MSE (MMSE), 
standard-deviation MSE (SMSE), Fr\'echet Poseidon distance (FPD), and constraint error (CE).
MMSE and SMSE compare the generated and reference field moments, whereas FPD compares their feature distributions with a pretrained Poseidon encoder~\citep{herde2024poseidon}.
For a constraint type $*\in\{\mathrm{IC},\mathrm{BC},\mathrm{CL}\}$, CE is the residual $\ell_2$ norm over the relevant domain, averaged over $N$ generated samples:
\[
  \mathrm{CE}(*) = \frac{1}{N}\sum_{n=1}^{N}
  \left\lVert \mathcal{R}_{*}\!\left(\hat{u}^{(n)}\right) \right\rVert_2.
\]
Lower values indicate closer agreement with the reference distribution or constraint.
The shared baselines follow the left-to-right family order in Table~\ref{tab:method-capabilities}, ending with \method{}.
Baselines FFM through PCFM report values from the established benchmark suite~\citep{utkarsh2025pcfm}.
PBFM~\citep{baldan2025pbfm} was retrained and evaluated under the identical benchmark protocol and data splits; its Navier--Stokes entry is unavailable because the upstream setup does not support the canonical 3D configuration.
The final column reports \method{} evaluated under the same independent test sets and protocol with one-step ($NFE=1$) deterministic physical decoding.

\begin{table*}[t]
\centering
\caption{Zero-shot constrained-PDE benchmark. Heat and Navier--Stokes: linear IC/CL; Reaction--Diffusion and Burgers: nonlinear CL with IC/BC. Lower is better; best baseline in bold.}
\label{tab:zero-shot-constrained-pde}
\tablefont
\setlength{\tabcolsep}{0.25pt}
\renewcommand{\arraystretch}{1.02}
\begin{tabularx}{\textwidth}{@{}>{\raggedright\arraybackslash}p{0.14\textwidth}>{\raggedright\arraybackslash}p{0.17\textwidth}Z>{\centering\arraybackslash}p{0.10\textwidth}Z>{\centering\arraybackslash}p{0.07\textwidth}*{3}{Z}>{\centering\arraybackslash}p{0.15\textwidth}@{}}
\toprule
Dataset & Metric & FFM & \shortstack{Diffusion\\PDE} & PDM & D-Flow & ECI & PCFM & PBFM & PMosFM \\
\midrule
\multirow{5}{*}{Heat Equation}
& MMSE / $10^{-2}$ & 4.56 & 4.49 & 0.45 & 1.97 & 0.697 & 0.241 & 0.043 & \textbf{0.016} \\
& SMSE / $10^{-2}$ & 3.51 & 3.93 & \textbf{0.02} & 1.14 & 0.973 & 0.937 & 0.242 & \textbf{0.003} \\
& CE (IC) / $10^{-2}$ & 579 & 599 & \textbf{0} & 102 & \textbf{0} & \textbf{0} & 24.2 & \textbf{0} \\
& CE (CL) / $10^{-2}$ & 2.11 & 2.06 & \textbf{0} & 64.8 & \textbf{0} & \textbf{0} & 57.5 & \textbf{0} \\
& FPD & 1.77 & 1.70 & 0.17 & 2.70 & 1.34 & 1.22 & -- & \textbf{0.0806} \\
\midrule
\multirow{5}{*}{Navier--Stokes}
& MMSE / $10^{-2}$ & 16.5 & 17.4 & 12.21 & \pending & 5.23 & 4.59 & -- & \textbf{4.41} \\
& SMSE / $10^{-2}$ & 7.90 & 9.48 & 6.61 & \pending & 7.28 & 4.17 & -- & \textbf{3.24} \\
& CE (IC) / $10^{-2}$ & 328 & 288 & \textbf{0} & \pending & \textbf{0} & \textbf{0} & -- & \textbf{0} \\
& CE (CL) / $10^{-2}$ & 18.6 & 21.4 & \textbf{0} & \pending & \textbf{0} & \textbf{0} & -- & \textbf{0} \\
& FPD & 2.81 & 3.70 & 2.01 & \pending & 1.04 & \textbf{1.00} & -- & 1.01 \\
\midrule
\multirow{5}{*}{\shortstack[l]{Reaction--\\Diffusion IC}}
& MMSE / $10^{-2}$ & 2.92 & 3.16 & 1.74 & 0.318 & 0.324 & 0.026 & \textbf{0.018} & 0.042 \\
& SMSE / $10^{-2}$ & 2.54 & 2.54 & 0.43 & 6.86 & 0.060 & 0.583 & 0.017 & \textbf{0.014} \\
& CE (IC) / $10^{-2}$ & 445 & 451 & \textbf{0} & 215 & \textbf{0} & \textbf{0} & 33.4 & \textbf{0} \\
& CE (CL) / $10^{-2}$ & 3.87 & 3.82 & \textbf{0} & 29.7 & 6.00 & \textbf{0} & 25.3 & \textbf{0} \\
& FPD & 24.9 & 44.1 & 109 & 28.3 & 136 & \textbf{15.7} & -- & 26.3 \\
\midrule
\multirow{5}{*}{Burgers BC}
& MMSE / $10^{-2}$ & 4.86 & 5.42 & 11.8 & 0.224 & 0.359 & 0.335 & 1.49 & \textbf{0.120} \\
& SMSE / $10^{-2}$ & 1.38 & 1.30 & 2.67 & 0.948 & 0.089 & 0.123 & 4.63 & \textbf{0.0396} \\
& CE (BC) / $10^{-2}$ & 409 & 426 & \textbf{0} & 95.7 & 20.3 & \textbf{0} & 102.3 & \textbf{0} \\
& CE (CL) / $10^{-2}$ & 6.91 & 6.20 & \textbf{0} & 15.0 & 15.7 & \textbf{0} & 35.5 & \textbf{0} \\
& FPD & 24.7 & 25.9 & 6.41 & 1.44 & 0.307 & \textbf{0.292} & 39.7 & 0.409 \\
\midrule
\multirow{5}{*}{Burgers IC}
& MMSE / $10^{-2}$ & 13.7 & 14.3 & 153 & 9.97 & 10.0 & 0.052 & 0.112 & \textbf{0.030} \\
& SMSE / $10^{-2}$ & 7.90 & 8.06 & 2.62 & 7.91 & 6.65 & 0.272 & 0.439 & \textbf{0.019} \\
& CE (IC) / $10^{-2}$ & 462 & 471 & \textbf{0} & 397 & \textbf{0} & \textbf{0} & 47.4 & \textbf{0} \\
& CE (CL) / $10^{-2}$ & 6.91 & 6.22 & \textbf{0} & 8.66 & 205 & \textbf{0} & 18.6 & \textbf{0} \\
& FPD & 33.5 & 35.8 & 99.7 & 22.1 & 1.31 & \textbf{0.101} & 2.20 & 0.283 \\
\bottomrule
\end{tabularx}
\parbox{\textwidth}{\footnotesize\raggedright
D-Flow and PBFM Navier--Stokes and PBFM Heat/RD FPD are omitted (instabilities / no 1D protocol).}
\end{table*}

\paragraph{Exact physical constraint preservation.}
Across five benchmark configurations and ten constraint evaluations, \method{} reports numerical residuals below $10^{-5}$ for initial, boundary, and conservation constraints.
For an exact decoder, the chart $\chi_c$ enforces the encoded discrete constraint within its valid domain.
In contrast, physics-loss and penalty-based methods such as PBFM accumulate substantial residual violations along their trajectories; unconstrained flow and diffusion baselines drift above $1$.

\paragraph{Distributional fidelity beyond feasibility.}
Across the evaluated equations, \method{} maintains strong agreement with the
reference field statistics while preserving the encoded physical constraints.
The consistent behavior of the moment-based metrics indicates that feasible-coordinate
generation improves residual accuracy while allowing the learned transport to
recover the dominant variability of the target distribution. Feature-space
distances are more task dependent, so exact feasibility and distributional
matching provide complementary measures of sample quality.

\paragraph{Task-dependent trade-offs.}
Reaction--Diffusion reveals complementary behavior across metrics: \method{}
achieves improved constraint satisfaction and variance statistics alongside
competitive mean-field accuracy. Similar variation in feature-space distances
across tasks reinforces this metric dependence. Overall, \method{} preserves
strict feasibility while remaining competitive across complementary measures
of distributional fidelity.

\subsubsection{Burgers with fixed initial condition}
\label{app:zero-shot-burgers-ic}

The Burgers IC rows of Table~\ref{tab:zero-shot-constrained-pde} are complemented by two visual diagnostics.
Fig.~\ref{fig:burgers-fixed-ic-conservation} compares field moments with mass and boundary residuals, while Fig.~\ref{fig:burgers-fixed-ic-solution} shows the evolving mean profiles and spread.
Both PCFM and PMosFM have small mass and boundary residuals in
Fig.~\ref{fig:burgers-fixed-ic-conservation}, but their field statistics differ.
PMosFM more closely follows the reference mean profiles and
standard-deviation pattern in Fig.~\ref{fig:burgers-fixed-ic-solution}.

\begin{figure*}[t]
\centering
\includegraphics[width=\textwidth]{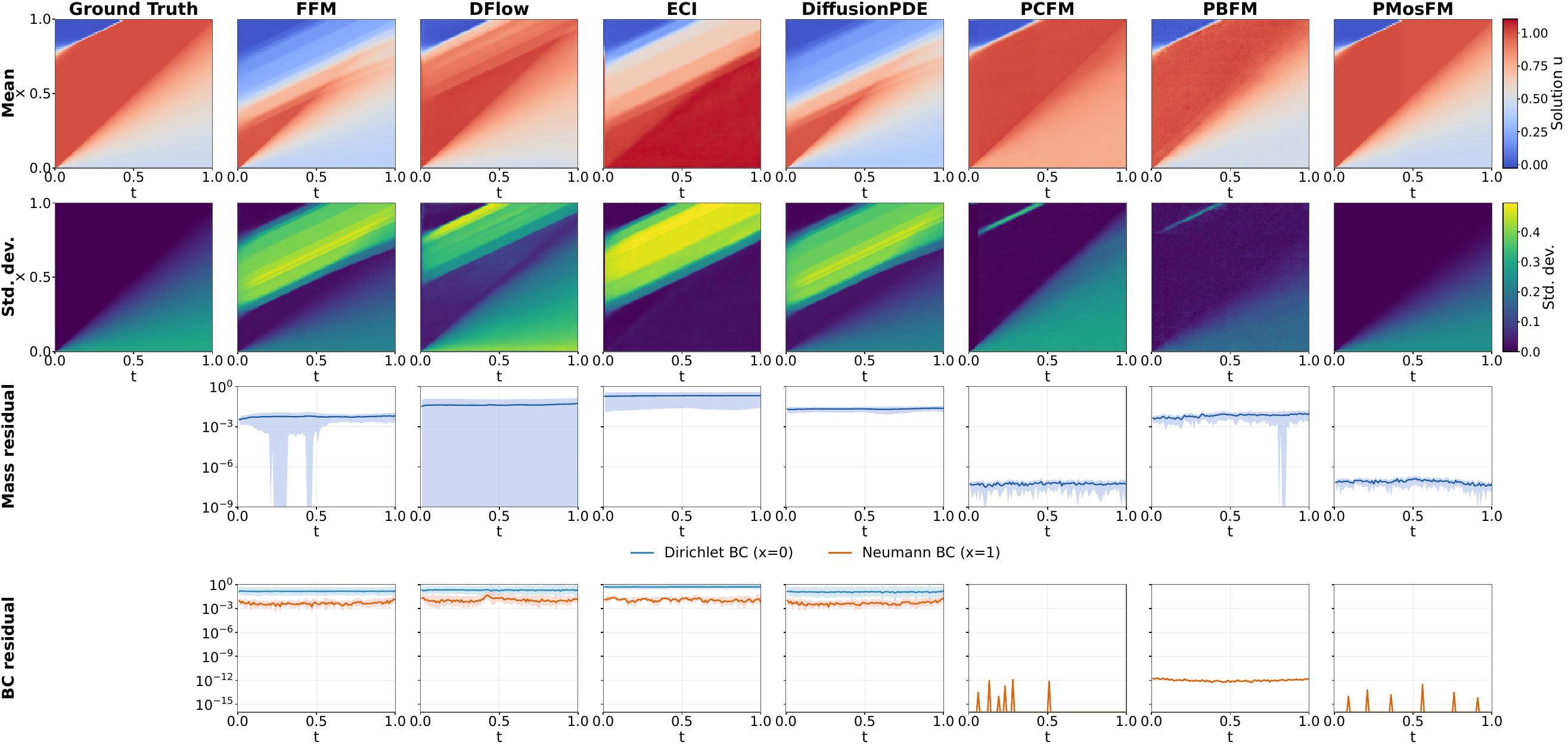}
\caption{Fixed-IC Burgers conservation diagnostics. Rows compare solution means and standard deviations, followed by mass and boundary residuals for the displayed methods.}
\label{fig:burgers-fixed-ic-conservation}
\end{figure*}

\begin{figure*}[t]
\centering
\includegraphics[width=\textwidth]{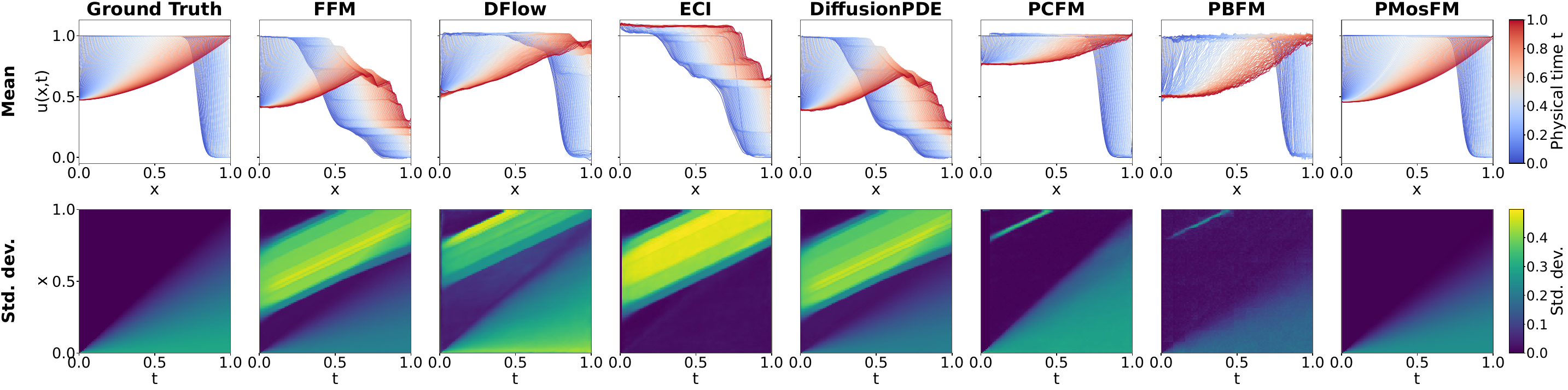}
\caption{Fixed-IC Burgers solution statistics. Mean solution profiles across time and pointwise standard deviations for the reference and displayed methods.}
\label{fig:burgers-fixed-ic-solution}
\end{figure*}

\subsection{Feasible-set and target-law study}
\label{app:controlled-support-measure}

We evaluate the closed-form ellipse construction of Section~\ref{sec:ablation} over five independent seeds and $24{,}000$ samples per seed.
In Table~\ref{tab:residual-manifold}, both exact parameterizations give
residual RMS near $10^{-16}$, but the co-area law lowers KL from
$0.369$ to $0.00207$ relative to the volume law.
Thus, exact feasibility alone does not recover the target law.

\begin{table}[H]
\centering
\caption{Residual-manifold support and measure over five independent seeds (24,000 samples each). Exact parameterizations yield near-zero residual; co-area correction recovers the target measure.}
\label{tab:residual-manifold}
\tablefont
\setlength{\tabcolsep}{4pt}
\begin{adjustbox}{max width=\textwidth}
\begin{tabular}{lccc}
\toprule
Method & Residual RMS $\downarrow$ & KL$(q\Vert q^\star)\downarrow$ & TV$(q,q^\star)\downarrow$ \\
\midrule
Ambient soft residual & $7.120\!\times\!10^{-2}\pm2.264\!\times\!10^{-4}$ & $2.077\!\times\!10^{-3}\pm1.548\!\times\!10^{-4}$ & $2.358\!\times\!10^{-2}\pm8.234\!\times\!10^{-4}$ \\
Exact parameterization, volume law & $1.917\!\times\!10^{-16}\pm1.326\!\times\!10^{-18}$ & $3.689\!\times\!10^{-1}\pm2.069\!\times\!10^{-3}$ & $3.723\!\times\!10^{-1}\pm1.229\!\times\!10^{-3}$ \\
Exact parameterization, co-area law & $1.195\!\times\!10^{-16}\pm1.646\!\times\!10^{-18}$ & $2.071\!\times\!10^{-3}\pm1.448\!\times\!10^{-4}$ & $2.468\!\times\!10^{-2}\pm1.634\!\times\!10^{-3}$ \\
\bottomrule
\end{tabular}
\end{adjustbox}
\end{table}

\paragraph{Construction and representation scope.}
Exact support fixes the feasible set, whereas the constrained measure specifies the probability law on that set.
The controlled feasible set is the ellipse $\chi(\theta)=(1.65\cos\theta,0.72\sin\theta)$ with scalar residual
\begin{equation*} R(x)=\exp(1.2\cos\theta(x))
 \left(\frac{x_1^2}{1.65^2}+\frac{x_2^2}{0.72^2}-1\right),
\end{equation*}
where $\theta(x)$ is the ellipse angle.
The residual scaling preserves the zero set while making $\norm{J_R}$ nonuniform.
For a tubular ambient density constant on the ellipse, the co-area target is $q^\star(\theta)\propto\norm{\partial_\theta\chi}/\norm{J_R}$, while a volume-only exact sampler has $q_{\mathrm{vol}}(\theta)\propto\norm{\partial_\theta\chi}$.
For an encoded exact residual, the controlled objective in Eq.~\ref{eq:pmosfm-loss} retains the physical matching terms after residual factorization removes the corresponding residual-normal correction.
The held-out physical studies separately examine encode--decode error, encoded residual, and independently evaluated residual.

\begin{wraptable}[8]{r}{0.56\textwidth}
\centering
\caption{Population tangent-regression conditioning with 64 eigenmodes and the same stable step rule. Lower is better.}
\label{tab:preconditioning-convergence}
\tablefont
\setlength{\tabcolsep}{2pt}
\renewcommand{\arraystretch}{1.0}
\resizebox{\linewidth}{!}{%
\begin{tabular}{lcc}
\toprule
Coordinates & Active $\kappa_+$ & Steps to $90\%$ $\downarrow$ \\
\midrule
Raw covariance coordinates & $10^{6}$ & $2{,}558{,}427$ \\
Exact state-whitened coordinates & $1$ & $1$ \\
\bottomrule
\end{tabular}
}%
\end{wraptable}

\subsection{Held-out diagnostics}
\label{app:chart-diagnostics}

This held-out pre-training study reports encode--decode field error together with decoder-side and independently evaluated residuals. The evaluation is performed with an evaluator that is independent of the coordinate decoder.
The manifold-conditioned Dynamic Stall case closes the specified algebraic constraint under both evaluators.
For Kolmogorov flow, the central-difference implementation has external RMS $1.353\times10^{-1}$, whereas the spectral implementation has external RMS $2.422\times10^{-6}$, with both implementations retaining relative $\ell_2$ error near $0.92$.

\subsection{Tangent-regression conditioning}
\label{app:controlled-conditioning}

Table~\ref{tab:controlled-mesh-conditioning} and the left panel of
Fig.~\ref{fig:preconditioning-convergence} show that grid refinement
worsens raw conditioning, while the ideal residual-space metric keeps
the active condition number near one.
The population study uses $64$ eigenmodes with a fixed step rule.
Table~\ref{tab:preconditioning-convergence} and the right panel of
Fig.~\ref{fig:preconditioning-convergence} show that exact whitening
removes covariance anisotropy and accelerates convergence in the
stated population-regression setting.

\subsection{Dynamic Stall field and ensemble results}
\label{app:dynamic-stall-results}

Fig.~\ref{fig:dynamic-stall-flow} compares the methods for one held-out Dynamic Stall condition.
The PMosFM row places the shock-associated transition at the same location across the pressure, velocity-gradient, temperature, density, and wall-shear channels as the reference field.

\begin{figure*}[t]
  \centering
  \includegraphics[height=.85\textheight,keepaspectratio]{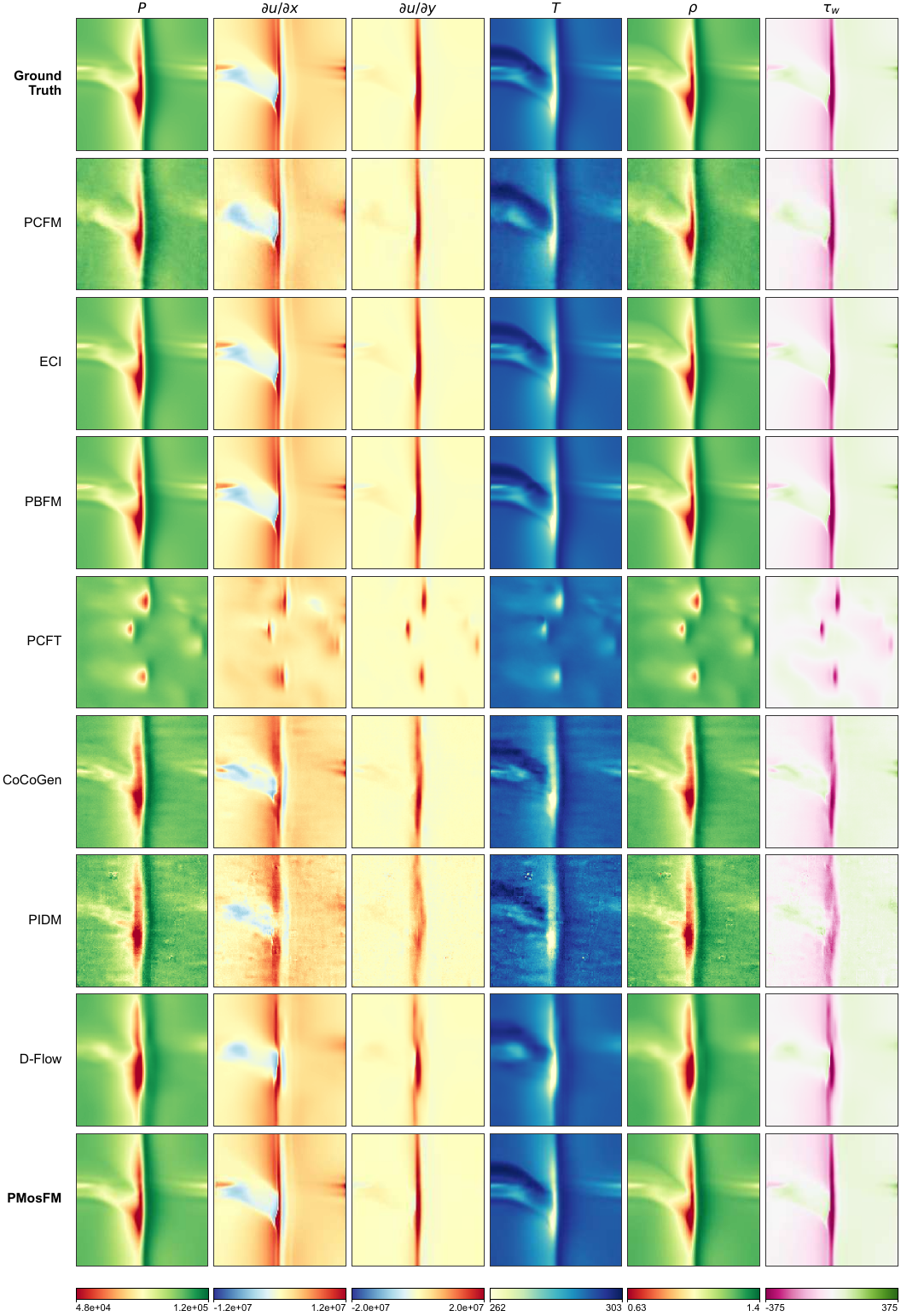}
  \caption{Qualitative Dynamic Stall comparison. One held-out condition is shown. Columns show pressure, velocity gradients, temperature, density, and wall shear. Rows compare the reference and model predictions; PMosFM retains the shock-transition location and cross-channel structure.}
  \label{fig:dynamic-stall-flow}
\end{figure*}

Fig.~\ref{fig:dynamic-stall-ensemble-moments} compares the
conditional ensemble means and standard deviations of the six generated channels
with the reference statistics.
The PMosFM mean preserves the dominant shock-associated structure,
but its standard-deviation maps show weaker fluctuations than
the reference under the shared color scales.

\begin{figure*}[t]
  \centering
  \includegraphics[width=0.98\textwidth]{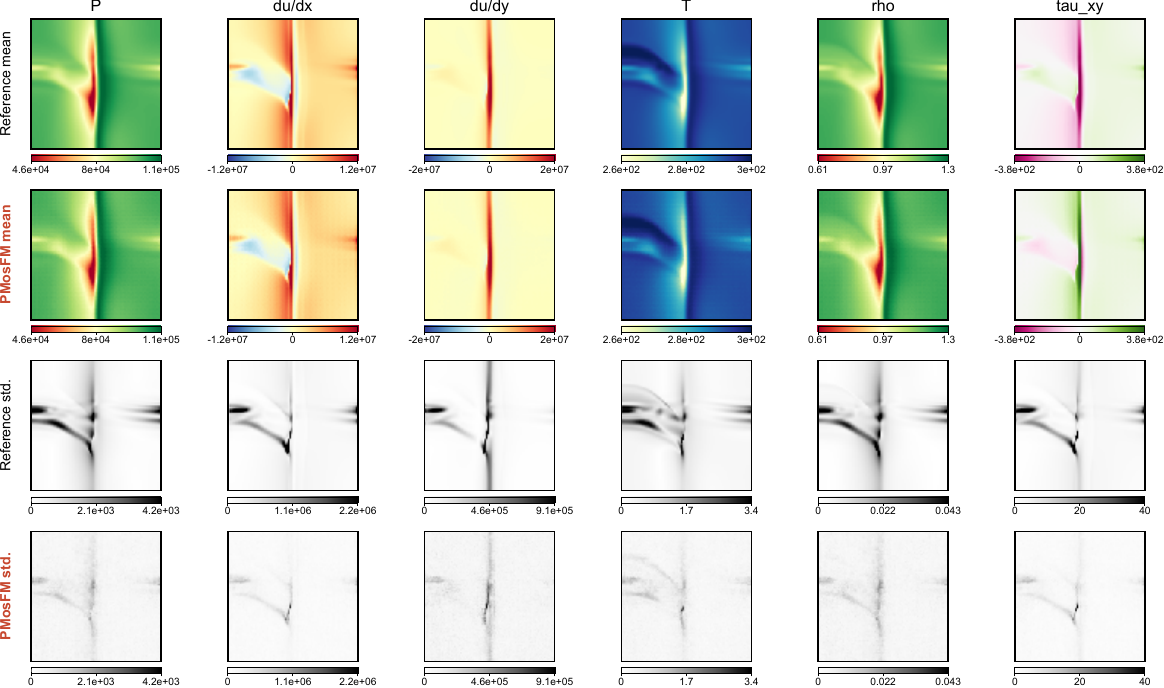}
  \caption{Dynamic Stall ensemble moments. Conditional means and standard deviations of pressure, tangential velocity gradients, temperature, density, and wall shear for one held-out condition.}
  \label{fig:dynamic-stall-ensemble-moments}
\end{figure*}

\subsection{Held-out one-step turbulence}
\label{app:heldout-turbulence-one-step}

The $Q$--$R$ evaluator uses periodic Fourier differentiation followed by trace removal.
Let $A$ denote the resulting velocity-gradient tensor, with
$Q=-\frac12\mathop{\mathrm{tr}}\nolimits(A^2)$ and $R=-\det(A)$.
Fixed histogram bins and the same normalization are used for all methods.
Table~\ref{tab:hit-charted} shows that PMosFM has the lowest field MSE,
divergence RMS, and all three $Q$--$R$ distribution distances among
the evaluated methods.

\begin{figure*}[t]
\centering
\includegraphics[width=\textwidth]{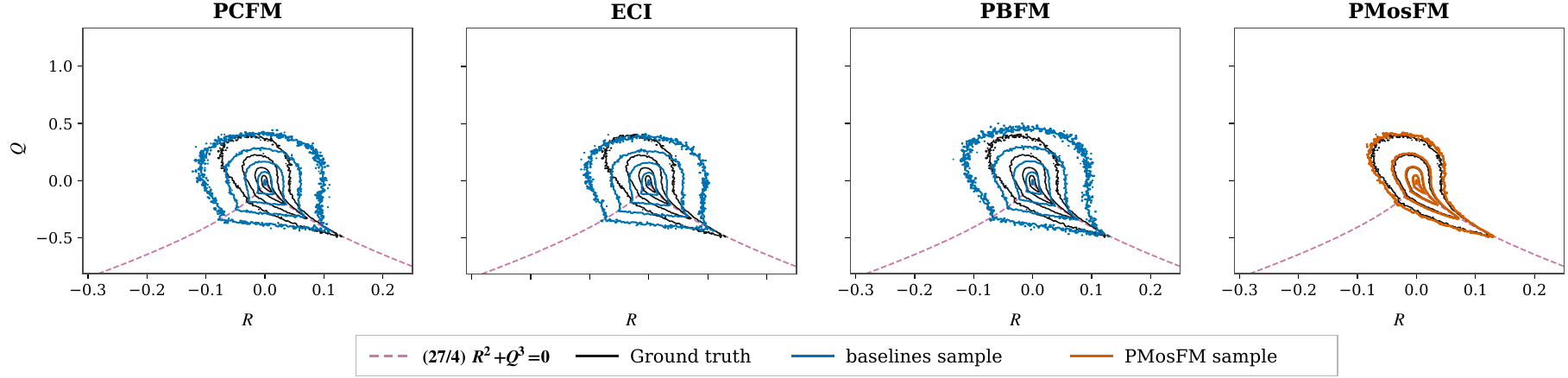}
\caption{Focused $Q$--$R$ comparison for turbulence forecasting.
Selected physics-constrained baselines are compared with the held-out ground truth.}
\label{fig:hit-forecasting-comparison}
\end{figure*}

\begin{table}[t]
\centering
\caption{Held-out turbulence evaluation for physics-constrained generation uses a common evaluator with fixed $Q$--$R$ bins. All methods follow the main-text objective; lower is better.}
\label{tab:hit-charted}
\small
\setlength{\tabcolsep}{4.1pt}
\begin{tabularx}{\textwidth}{@{}Y Z Z Z Z Z@{}}
\toprule
Method & Field MSE $\downarrow$ & \makebox[0pt][c]{Divergence RMS $\downarrow$} & $Q$--$R$ JS $\downarrow$ & $Q$--$R$ TV $\downarrow$ & Hellinger $\downarrow$ \\
\midrule
FM-OT & $1.097$ & $1.000\times10^{-1}$ & $9.165\times10^{-3}$ & $7.499\times10^{-2}$ & $1.048\times10^{-1}$ \\
advNO & $2.205$ & $1.248$ & $1.317\times10^{-1}$ & $4.171\times10^{-1}$ & $3.798\times10^{-1}$ \\
FluidFlow & $1.102$ & $1.004\times10^{-1}$ & $9.422\times10^{-3}$ & $7.609\times10^{-2}$ & $1.062\times10^{-1}$ \\
CoCoGen & $2.006\times10^{2}$ & $4.693$ & $6.803\times10^{-2}$ & $2.857\times10^{-1}$ & $2.727\times10^{-1}$ \\
DiffusionPDE & $1.785$ & $6.874\times10^{-1}$ & $1.815\times10^{-1}$ & $4.653\times10^{-1}$ & $4.617\times10^{-1}$ \\
PIDM & $1.695\times10^{2}$ & $4.248$ & $6.351\times10^{-2}$ & $2.721\times10^{-1}$ & $2.638\times10^{-1}$ \\
D-Flow & $2.048\times10^{2}$ & $4.676$ & $6.107\times10^{-2}$ & $2.664\times10^{-1}$ & $2.588\times10^{-1}$ \\
ECI & $6.472\times10^{-2}$ & $1.426\times10^{-1}$ & $8.684\times10^{-2}$ & $3.360\times10^{-1}$ & $3.041\times10^{-1}$ \\
PCFM & $6.357\times10^{-2}$ & $7.464\times10^{-2}$ & $8.099\times10^{-2}$ & $3.227\times10^{-1}$ & $2.936\times10^{-1}$ \\
PCFT & $1.096$ & $1.004\times10^{-1}$ & $9.230\times10^{-3}$ & $7.338\times10^{-2}$ & $1.053\times10^{-1}$ \\
PBFM & $4.892\times10^{-2}$ & $3.387\times10^{-2}$ & $5.793\times10^{-2}$ & $2.708\times10^{-1}$ & $2.478\times10^{-1}$ \\
\midrule
\textbf{PMosFM} & $\bm{5.359\times10^{-3}}$ & $\bm{1.527\times10^{-7}}$ & $\bm{5.961\times10^{-3}}$ & $\bm{5.293\times10^{-2}}$ & $\bm{8.513\times10^{-2}}$ \\
\bottomrule
\end{tabularx}
\end{table}

Figs.~\ref{fig:hit-forecasting-comparison} and~\ref{fig:hit-forecasting-qr}
compare generated and held-out $Q$--$R$ laws under the common evaluator.
PMosFM follows the reference contours more closely than PCFM, ECI, and PBFM,
while FM-OT, FluidFlow, and PCFT also retain much of the contour structure
in the full-baseline comparison.

\begin{figure}[t]
\centering
\includegraphics[width=\textwidth]{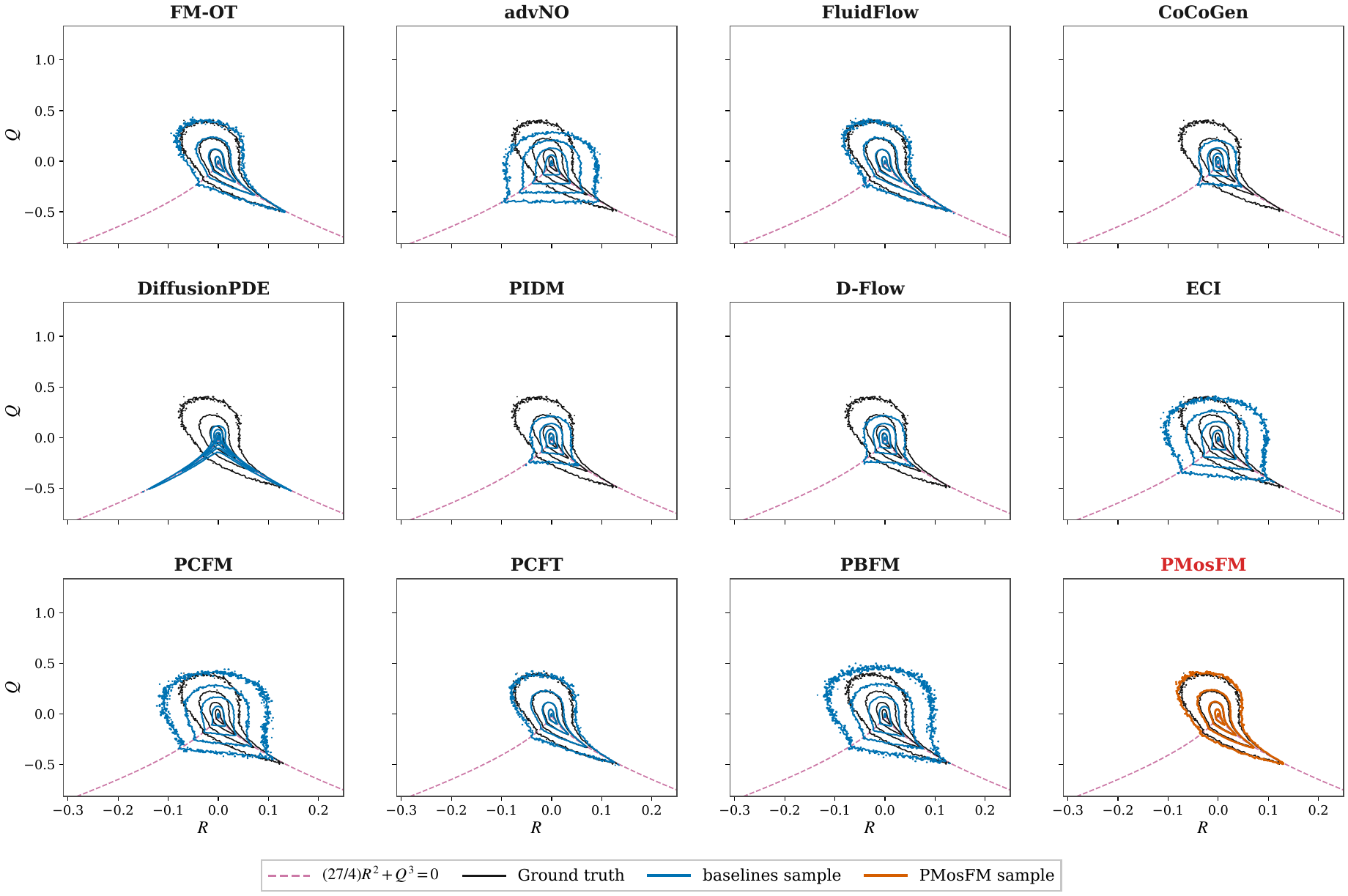}
\caption{Full $Q$--$R$ comparison for turbulence forecasting.
All evaluated methods are compared with the held-out ground truth.}
\label{fig:hit-forecasting-qr}
\end{figure}

\subsection{Qualitative turbulence diagnostics}

\makeatletter
\setlength{\@dblfptop}{0pt}
\makeatother
\begin{table*}[!t]
\centering
\caption{TFR sparse-reconstruction metrics for held-out frames 168--186 (32 samples per condition). Lower is better except interval width; bold marks the minimum at each missing rate.}
\label{tab:tfr-sparse-reconstruction}
\setlength{\tabcolsep}{2.4pt}
\renewcommand{\arraystretch}{1.05}
\resizebox{0.75\linewidth}{!}{%
\begin{tabular}{@{}clccccc@{}}
\toprule
\multirow[c]{2}{*}{Missing} & \multirow[c]{2}{*}{Method} & Missing-region & Physical & \multirow[c]{2}{*}{CRPS $\downarrow$} & Coverage & Interval \\
& & RMSE $\downarrow$ & residual $\downarrow$ & & error $\downarrow$ & width \\
\midrule
\multirow[c]{5}{*}{50\%}
& PBFM & 0.063159 & 1.102126 & 0.024882 & 0.075232 & 0.084228 \\
& ECI & 0.045158 & 1.910683 & 0.019637 & 0.035143 & 0.126545 \\
& PCFM & 0.067615 & 1.520453 & 0.030269 & 0.099867 & 0.112164 \\
& PCFT & 0.069091 & 1.081061 & 0.027396 & 0.095131 & 0.082924 \\
& PMosFM & \textbf{0.029653} & \textbf{0.620392} & \textbf{0.008742} & \textbf{0.010382} & 0.036809 \\
\midrule
\multirow[c]{5}{*}{90\%}
& PBFM & 0.151450 & 0.677789 & 0.075305 & 0.371066 & 0.090451 \\
& ECI & 0.066441 & 1.867086 & 0.028625 & 0.004689 & 0.158808 \\
& PCFM & 0.089082 & 1.211505 & 0.038495 & 0.070901 & 0.148736 \\
& PCFT & 0.152720 & 0.698189 & 0.075118 & 0.351275 & 0.088116 \\
& PMosFM & \textbf{0.045076} & \textbf{0.628506} & \textbf{0.014650} & \textbf{0.036005} & \textbf{0.059326} \\
\midrule
\multirow[c]{5}{*}{99\%}
& PBFM & 0.187398 & 1.594082 & 0.095594 & 0.430716 & 0.092738 \\
& ECI & 0.165424 & 1.695730 & 0.078446 & 0.173864 & 0.192676 \\
& PCFM & 0.177361 & 1.103668 & 0.085315 & 0.219686 & 0.171779 \\
& PCFT & 0.187990 & 0.622136 & 0.095010 & 0.416529 & 0.090029 \\
& PMosFM  & \textbf{0.075462} & \textbf{0.53680} & \textbf{0.026655} & \textbf{0.021873} & 0.119706 \\
\bottomrule
\end{tabular}%
}
\end{table*}

Fig.~\ref{fig:tfr-mask-sensitivity} shows that reconstruction becomes
harder as observations become sparser.
At $99\%$ missingness, PMosFM retains a coarse spatial pattern, but
its spectra and gradient PDFs still differ from the reference.

\begin{figure*}[h]
\centering
\includegraphics[width=.95\textwidth]{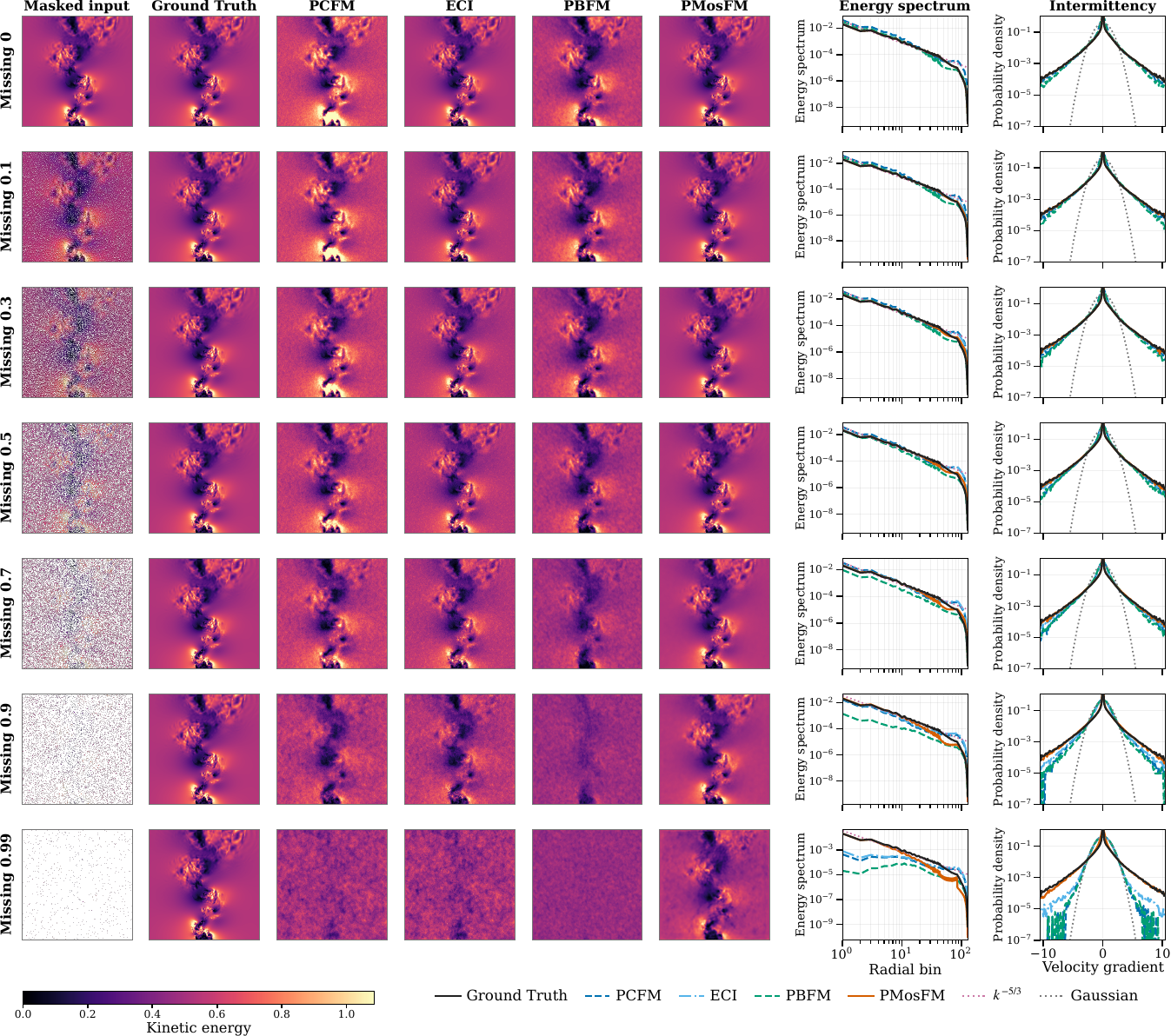}
\caption{Random-mask reconstruction. One held-out turbulence frame is tested with shared exact-count velocity masks (0--99\% missing). Columns show the input, ground truth, PCFM, ECI, PBFM, and PMosFM; lower panels show energy spectra and standardized $\partial_z w$ PDFs.}
\label{fig:tfr-mask-sensitivity}
\end{figure*}

The held-out HIT benchmark evaluates one-step generation, while the random-mask test assesses sparse reconstruction.
PMosFM attains the lowest field MSE, divergence RMS, and $Q$--$R$ distribution distances in the HIT table; in reconstruction, it has the lowest missing-region RMSE at all three rates and retains an identifiable large-scale pattern at $99\%$ missingness (Tables~\ref{tab:hit-charted} and~\ref{tab:tfr-sparse-reconstruction}; Fig.~\ref{fig:tfr-mask-sensitivity}).

\section{Additional Ablations}
\label{app:required-full-objective-ablation}

\paragraph{Loss balance.}
Table~\ref{tab:additional-ablation-gamma} tests whether velocity supervision and endpoint matching play complementary roles. On both Kolmogorov and Darcy flow, the balanced weight gives the lowest field error and physical sliced $W_2$, together with a near-zero reported residual. Moving toward either single-term objective worsens all three metrics, suggesting that the two losses are complementary.

\begin{table}[H]
\centering
\caption{Velocity--endpoint loss balance on Kolmogorov and Darcy flow.}
\label{tab:additional-ablation-gamma}
\small
\setlength{\tabcolsep}{3pt}
\resizebox{\textwidth}{!}{%
\begin{tabular}{lllccc}
\toprule
\textbf{Benchmark} & \textbf{Configuration} & \textbf{Loss Formulation} & \textbf{Field MSE} $\downarrow$ & \textbf{Residual / Div. RMS} $\downarrow$ & \textbf{Physical Sliced $W_2$} $\downarrow$ \\
\midrule
\multirow{5}{*}{\shortstack[l]{\textbf{Kolmogorov Flow}\\($128\times128$)}}
& PE only ($\gamma=0.0$) & Direct Endpoint $\mathcal L_{\mathrm{PE}}$ & $0.4241 \pm 0.0182$ & $1629.9 \pm 52.4$ & $0.3852 \pm 0.0124$ \\
& PE-heavy ($\gamma=0.25$) & Two-Time Joint & $0.1853 \pm 0.0118$ & $850.4 \pm 38.1$ & $0.2821 \pm 0.0092$ \\
& \textbf{PMosFM Balanced ($\gamma=0.50$)} & \textbf{Two-Time Joint ($\mathcal L_{\mathrm{FM}}+\mathcal L_{\mathrm{PE}}$)} & $\mathbf{0.0531 \pm 0.0019}$ & $\mathbf{3.45\times10^{-9}}$ & $\mathbf{0.2522 \pm 0.0014}$ \\
& FM-heavy ($\gamma=0.75$) & Two-Time Joint & $0.2184 \pm 0.0145$ & $1453.1 \pm 46.2$ & $0.2789 \pm 0.0101$ \\
& FM only ($\gamma=1.0$) & Velocity Flow Matching $\mathcal L_{\mathrm{FM}}$ & $0.4714 \pm 0.0210$ & $7258.3 \pm 112.5$ & $0.3621 \pm 0.0153$ \\
\midrule
\multirow{5}{*}{\shortstack[l]{\textbf{Darcy Flow}\\($64\times64$)}}
& PE only ($\gamma=0.0$) & Direct Endpoint $\mathcal L_{\mathrm{PE}}$ & $0.0382 \pm 0.0024$ & $870.8 \pm 42.1$ & $0.0284 \pm 0.0018$ \\
& PE-heavy ($\gamma=0.25$) & Two-Time Joint & $0.0115 \pm 0.0008$ & $128.4 \pm 9.5$ & $0.0098 \pm 0.0006$ \\
& \textbf{PMosFM Balanced ($\gamma=0.50$)} & \textbf{Two-Time Joint ($\mathcal L_{\mathrm{FM}}+\mathcal L_{\mathrm{PE}}$)} & $\mathbf{0.0044 \pm 0.0001}$ & $\mathbf{(1.41 \pm 0.03)\times10^{-12}}$ & $\mathbf{0.0044 \pm 0.0001}$ \\
& FM-heavy ($\gamma=0.75$) & Two-Time Joint & $0.0092 \pm 0.0007$ & $240.1 \pm 14.2$ & $0.0087 \pm 0.0005$ \\
& FM only ($\gamma=1.0$) & Velocity Flow Matching $\mathcal L_{\mathrm{FM}}$ & $0.0415 \pm 0.0031$ & $1577.2 \pm 68.3$ & $0.0312 \pm 0.0021$ \\
\bottomrule
\end{tabular}
}
\par\smallskip
{\footnotesize\raggedright\textit{Note.} Results are mean $\pm$ standard deviation over five independent seeds.\par}
\end{table}

\paragraph{Endpoint-gradient term.}
The held-out turbulence implementation is averaged over five independently trained model states and five independent inference seeds, whereas external methods are single-state estimates.
The held-out turbulence implementation has $948{,}868$ parameters and performs one neural evaluation followed by physical decoding.
It trains with direct velocity MSE, decoded-endpoint MSE, and a $0.2$-weighted endpoint-gradient MSE under the hard projector.
Table~\ref{tab:hit-gradient-ablation} isolates the endpoint-gradient term in this one-step implementation.
Velocity and Endpoint denote direct velocity MSE and decoded-endpoint MSE, respectively; $\lambda_{\nabla}$ is the endpoint-gradient MSE coefficient, and ``--'' denotes omission.
Both variants share a residual-consistent projector and are evaluated over $16$ validation windows, five independent training seeds, and five independent inference seeds.
Bold marks the lower value in each diagnostic column.
Table~\ref{tab:hit-gradient-ablation} shows a small reduction in physical MSE
($5.631\times10^{-3}$ to $5.524\times10^{-3}$) and lower $Q$--$R$
distributional distances, while divergence RMS remains near
$1.5\times10^{-7}$.

\begin{table}[t]
\centering
\caption{One-step objective ablation for one-step turbulence forecasting.
Lower is better.}
\label{tab:hit-gradient-ablation}
\small
\setlength{\tabcolsep}{3pt}
\renewcommand{\arraystretch}{1.05}
\begin{tabularx}{\textwidth}{@{}>{\raggedright\arraybackslash}p{0.12\textwidth} c c c Z Z Z Z Z@{}}
\toprule
& \multicolumn{3}{c}{Training terms} & \multicolumn{5}{c}{Held-out diagnostic} \\
\cmidrule(lr){2-4}\cmidrule(l){5-9}
Variant & Velocity & Endpoint & $\lambda_{\nabla}$ & Physical MSE $\downarrow$ & Divergence RMS $\downarrow$ & $Q$--$R$ JS $\downarrow$ & $Q$--$R$ TV $\downarrow$ & Hellinger $\downarrow$ \\
\midrule
No gradient & yes & yes & -- & $5.631\times10^{-3}$ & \textbf{$1.507\times10^{-7}$} & $6.257\times10^{-3}$ & $5.634\times10^{-2}$ & $8.698\times10^{-2}$ \\
Full objective & yes & yes & $0.2$ & \textbf{$5.524\times10^{-3}$} & $1.508\times10^{-7}$ & \textbf{$6.202\times10^{-3}$} & \textbf{$5.565\times10^{-2}$} & \textbf{$8.664\times10^{-2}$} \\
\bottomrule
\end{tabularx}
\end{table}

\paragraph{Preconditioning.}
Table~\ref{tab:additional-ablation-operator} compares identity, diagonal standardization, and preconditioning. In the controlled regression, preconditioning gives the lowest condition number and iteration count; it also gives the lowest reported Kolmogorov divergence error.

\begin{table}[H]
\centering
\caption{Geometric preconditioning effects on conditioning and Kolmogorov-flow convergence.}
\label{tab:additional-ablation-operator}
\small
\setlength{\tabcolsep}{3pt}
\resizebox{\textwidth}{!}{%
\begin{tabular}{lcccc}
\toprule
\textbf{Transform} & \textbf{Condition $\kappa(H)$} $\downarrow$ & \textbf{90\% Reduction} $\downarrow$ & \textbf{Kolmogorov Div Err} $\downarrow$ & \textbf{Iteration ratio} \\
\midrule
Identity (Vanilla) & $1.00\times10^6$ & $2{,}558{,}427$ & $625.17$ & $1.0\times$ (baseline) \\
Standardization (Diagonal) & $4.82\times10^3$ & $1{,}240$ & $485.30$ & $2{,}063\times$ \\
\textbf{Preconditioning (PMosFM)} & $\mathbf{1.00}$ & $\mathbf{1\text{ step}}$ & $\mathbf{234.01}$ & $\mathbf{>10^6\times}$ \\
\bottomrule
\end{tabular}%
}
\end{table}

\paragraph{Decoder iterations.}
Table~\ref{tab:additional-ablation-darcy-solver} examines the Darcy decoder's accuracy--cost trade-off. Higher $K$ lowers both errors but raises decoding latency, which is separate from one network evaluation.

\begin{table}[H]
\centering
\caption{Darcy decoder accuracy and latency across iteration counts.}
\label{tab:additional-ablation-darcy-solver}
\small
\setlength{\tabcolsep}{3pt}
\resizebox{0.85\linewidth}{!}{%
\begin{tabular}{cccc}
\toprule
\textbf{Iterations $K$} & \textbf{Relative Pressure Error} $\downarrow$ & \textbf{Independent Residual RE} $\downarrow$ & \textbf{Latency / sample (ms)} $\downarrow$ \\
\midrule
16 & $0.0124 \pm 0.0012$ & $0.0385 \pm 0.0035$ & \textbf{0.62 ms} \\
32 & $0.0031 \pm 0.0004$ & $0.0084 \pm 0.0009$ & 0.98 ms \\
64 & $0.0008 \pm 0.0001$ & $0.0019 \pm 0.0002$ & 1.65 ms \\
128 & $0.0002 \pm 0.00003$ & $0.00048 \pm 0.00005$ & 2.84 ms \\
256 & $\mathbf{0.00004 \pm 0.00001}$ & $\mathbf{0.00012 \pm 0.00002}$ & 5.12 ms \\
\bottomrule
\end{tabular}%
}
\end{table}

\section{Mathematical Setting and Assumptions}
\label{app:mathematical-setting}

This appendix records assumptions, proofs, and constructions for the discretized residual $R_h$, proceeding from residual factorization through preconditioning to two-time matching.

Fix a condition $c$ and suppress the condition index when no ambiguity arises.
Let $a\in\mathbb{R}^{s}$ denote either a generator input, a latent variable, or a vector of model parameters.
For a positive semidefinite matrix $M$, write $\norm{v}_{M}^{2}=v^\top Mv$.
For a matrix $A$, let $\sigma_{\min}^{+}(A)$ denote the smallest nonzero singular value of $A$ and define the active-subspace condition number
\begin{equation*}    \kappa_{+}(A)
    \coloneqq
    \frac{\sigma_{\max}(A)}{\sigma_{\min}^{+}(A)}.
\end{equation*}
This convention separates spectral stiffness from null directions imposed by a physical constraint.

\begin{assumption}[Regular residual factorization]
\label{ass:regular-manifold}
On the represented target region, $R_h(\cdot,c)$ is $C^1$ and $D_xR_h(x,c)$ has constant rank $q$.
The map $\chi_c:\mathcal{Y}_c\rightarrow\mathcal{X}_h$ is a $C^1$ local embedding of dimension $m=d-q$ on each active coordinate patch; the union of represented feasible sets covers the target region, and
\begin{equation}    R_h(\chi_c(y),c)=0
    \qquad\text{for every }y\in\mathcal{Y}_c.
    \label{eq:app-exact-chart}
\end{equation}
The represented component is $\mathcal M_c^\chi=\chi_c(\mathcal Y_c)\subseteq\mathcal Z_c$ from Eq.~\ref{eq:feasible-set} on each patch.
Every intrinsic model path remains inside a valid coordinate domain or is covered by compatible local parameterizations.
\end{assumption}

\begin{assumption}[Differentiability and finite moments]
\label{ass:differentiability}
The two-time map $T_\theta^{s,t}$ and the decoder $\Psi_c$ are differentiable in the variables under analysis.
The interpolation covariance defining $P_{s,c}$ is finite, symmetric, and positive semidefinite.
Whenever a second-order expansion is used, each map appearing in the expansion is $C^2$ on the corresponding line segment.
\end{assumption}

\subsection{Baseline constraint mechanisms}
\label{app:baseline-mechanisms}

Relative to the ambient objective in Eq.~\ref{eq:ambient_fm}, PBFM defines separate flow-matching and physical losses, denoted $\mathcal J_{\mathrm{FM}}$ and $\mathcal J_R$, and coordinates their gradients before terminal residual evaluation~\citep{baldan2025pbfm}:
\begin{equation*} g_{\mathrm{FM}}=\nabla_\phi\mathcal J_{\mathrm{FM}},\qquad
 g_R=\nabla_\phi\mathcal J_R,\qquad
 g_{\mathrm{update}}=\mathop{\mathrm{ConFIG}}\nolimits(g_{\mathrm{FM}},g_R).
\end{equation*}
PCFT instead applies weak-form fine-tuning~\citep{Tauberschmidt2025}.
DiffusionPDE adds sampling-time physics guidance~\citep{huang2024diffusionpde}; CCFM applies a chance-constrained projection~\citep{liang2025chance}.
ECI extrapolates a clean endpoint and applies a correction $\mathcal Q_c$ before mapping back to the current noise level~\citep{Cheng2025}:
\begin{equation*} \widehat x_{1\mid\tau}=x_\tau+(1-\tau)v_\phi(x_\tau,\tau,c),\qquad
 x_{1\mid\tau}^{c}=\mathcal Q_c(\widehat x_{1\mid\tau}).
\end{equation*}
PCFM applies nonlinear corrections to estimated terminal states during sampling~\citep{utkarsh2025pcfm}.
For a full-row-rank residual Jacobian $J_R$, a local Gauss--Newton correction contains
\begin{equation*} x_1^{\mathrm{GN}}=\widetilde x_1-J_R^\top(J_RJ_R^\top)^{-1}R_h(\widetilde x_1,c),
\end{equation*}
followed by method-specific back-mapping, relaxed updates, or a constrained solve when required.

\section{One-Step Sensitivity and Curvature}
\label{app:large-step-theory}

\subsection{Three notions of a step}

Training-time unrolling~\citep{baldan2025pbfm} evaluates a nonlinear residual after intermediate updates; inference-time integration counts network evaluations used to solve a learned ODE.
PMosFM learns a finite-interval map evaluated once at inference.
The Euler rollout in Eq.~\ref{eq:app-unrolled-map} serves as a conditioning diagnostic.

For an $n$-step Euler rollout from sampling time $t$, define
\begin{equation}\begin{aligned}
    \mathcal U_i(x)&=x+\Delta\tau\,v_\theta(x,\tau_i,c),
    &\Delta\tau&=\frac{1-t}{n},
    &\tau_i&=t+i\Delta\tau,\\
    T_{\theta,t\rightarrow1}^{(n)}
    &=\mathcal U_{n-1}\circ\cdots\circ\mathcal U_0.
\end{aligned}
\label{eq:app-unrolled-map}
\end{equation}
Let $A_i=D_xv_\theta(x^{(i)},\tau_i,c)$.
The chain rule gives
\begin{equation}    D_xT_{\theta,t\rightarrow1}^{(n)}
    =(I+\Delta\tau A_{n-1})\cdots(I+\Delta\tau A_0).
    \label{eq:app-unrolled-jacobian}
\end{equation}

\begin{proposition}[Finite-interval sensitivity]
\label{prop:app-flow-sensitivity}
Let $x(s)$ solve $\dot x(s)=v(x(s),s,c)$ for $s\in[t,1]$, and assume $A(s)=D_xv(x(s),s,c)$ is continuous.
The endpoint sensitivity $J_{t,s}=D_{x(t)}x(s)$ satisfies
\begin{equation}    \frac{dJ_{t,s}}{ds}=A(s)J_{t,s},
    \qquad J_{t,t}=I.
    \label{eq:app-variational-equation}
\end{equation}
The sensitivity also obeys
\begin{equation}    \norm{J_{t,1}}_2
    \leq
    \exp\!\left(\int_t^1\norm{A(s)}_2\,ds\right).
    \label{eq:app-sensitivity-bound}
\end{equation}
If the flow is locally invertible, then
\begin{equation}    \kappa_2(J_{t,1})
    \leq
    \exp\!\left(2\int_t^1\norm{A(s)}_2\,ds\right).
    \label{eq:app-sensitivity-condition-bound}
\end{equation}
\end{proposition}

\begin{proof}
Differentiating with respect to the initial state gives Eq.~\ref{eq:app-variational-equation}; Gronwall's inequality yields Eq.~\ref{eq:app-sensitivity-bound}.
Eq.~\ref{eq:app-unrolled-jacobian} discretizes Proposition~\ref{prop:app-flow-sensitivity}; the controlled study evaluates the conditioning of the one-step endpoint Jacobian.
The inverse sensitivity satisfies the corresponding backward variational equation, and the two norm bounds give Eq.~\ref{eq:app-sensitivity-condition-bound}.
\end{proof}

\subsection{Residual-normal curvature}

Let $G_\phi(a,c)$ be any differentiable endpoint generator and define
\begin{equation}    r_\phi(a,c)=R_h(G_\phi(a,c),c),
    \qquad
    \ell_R(a)=\frac{1}{2}\norm{r_\phi(a,c)}_{W_R}^{2},
    \quad W_R\succeq0.
    \label{eq:app-residual-objective}
\end{equation}
The residual $R_h$ has $n_R$ components.

\begin{lemma}[Exact Hessian decomposition]
\label{lem:app-hessian-decomposition}
Assume $R_h$ and $G_\phi$ are $C^2$.
Let $J_G=D_aG_\phi$ and $J_R=D_xR_h\vert_{G_\phi(a,c)}$.
Then
\begin{equation}\begin{aligned}
    \nabla_a\ell_R
    &=(J_RJ_G)^\top W_Rr_\phi,\\
    \nabla_a^2\ell_R
    &=(J_RJ_G)^\top W_R(J_RJ_G)
      +\sum_{j=1}^{n_R}(W_Rr_\phi)_j\,
        \nabla_a^2(r_\phi)_j.
\end{aligned}
\label{eq:app-exact-residual-hessian}
\end{equation}
At a feasible output, $r_\phi=0$, the second term vanishes and the exact Hessian equals the Gauss--Newton block
\begin{equation}    H_{\mathrm{GN},a}^{R}
    =J_G^\top J_R^\top W_RJ_RJ_G.
    \label{eq:app-composite-gn}
\end{equation}
\end{lemma}

\begin{proof}
Apply the chain rule to Eq.~\ref{eq:app-residual-objective}.
Differentiating the gradient gives the two terms in Eq.~\ref{eq:app-exact-residual-hessian}.
Feasibility eliminates the latter.
\end{proof}

\begin{corollary}[Squared active spectral spread]
\label{cor:app-squared-conditioning}
Let $B=W_R^{1/2}J_RJ_G$.
On the active subspace of $B$,
\begin{equation}    H_{\mathrm{GN},a}^{R}=B^\top B,
    \qquad
    \kappa_+(H_{\mathrm{GN},a}^{R})
    =\kappa_+(B)^2.
\label{eq:app-squared-condition}
\end{equation}
In addition,
\begin{equation*}    \lambda_{\max}(H_{\mathrm{GN},a}^{R})
    \leq
    \norm{W_R}_2\,
    \norm{J_R}_2^2\,
    \norm{J_G}_2^2.
\end{equation*}
\end{corollary}

\begin{proof}
Eq.~\ref{eq:app-squared-condition} follows by squaring the singular values of $B$; submultiplicativity bounds the norm.
\end{proof}

\begin{remark}[Condition-number product]
\label{rem:app-no-product-law}
The common heuristic $\kappa(H_{\mathrm{GN}})\approx\kappa(J_G)^2\kappa(J_R^\top J_R)$ depends on alignment between the relevant singular subspaces.
Corollary~\ref{cor:app-squared-conditioning} gives the unconditional result: the condition number of the composed residual differential is squared.
After generator composition, exact residual factorization also removes the residual block in Eq.~\ref{eq:app-composite-gn}.
\end{remark}

\section{Manifold Parameterization}
\label{app:manifold-theory}

\subsection{Residual-factorization differential}

\begin{lemma}[Residual-factorization differential]
\label{lem:app-tangent-space}
Under Assumption~\ref{ass:regular-manifold}, for every $y\in\mathcal{Y}_c$ and $x=\chi_c(y)$,
\begin{equation}    \begin{aligned}
    J_R(x,c)J_\chi(y,c)&=0,
    &\mathop{\mathrm{range}}\nolimits J_\chi(y,c)&\subseteq\mathop{\mathrm{ker}}\nolimits J_R(x,c),\\
    \mathop{\mathrm{range}}\nolimits J_\chi(y,c)&=\mathop{\mathrm{ker}}\nolimits J_R(x,c)
    =T_x\mathcal{M}_c^\chi.
    \end{aligned}
\label{eq:app-tangent-kernel}
\end{equation}
\end{lemma}

\begin{proof}
Differentiating $R_h(\chi_c(y),c)=0$ gives $J_RJ_\chi=0$, which implies $\mathop{\mathrm{range}}\nolimits J_\chi\subseteq\mathop{\mathrm{ker}}\nolimits J_R$.
Both spaces in Eq.~\ref{eq:app-tangent-kernel} have dimension $d-q$, so the inclusion is equality.
\end{proof}

\subsection{Exact endpoint feasibility}
\label{app:exact-manifold-proof}

Let $T_\phi^{0,1}(r_0,c)$ lie in the valid coordinate domain and set
$G_{\phi,c}^{\mathrm{RF}}=\Psi_c\circ T_\phi^{0,1}$.
Assumption~\ref{ass:regular-manifold} gives
\begin{equation}    R_h(G_{\phi,c}^{\mathrm{RF}}(r_0,c),c)
    =R_h\!\left(\chi_c\!\left(\mu_c+C_c^{-1}T_\phi^{0,1}(r_0,c)\right),c\right)=0.
    \label{eq:app-composed-feasibility}
\end{equation}
The left-hand side of Eq.~\ref{eq:app-composed-feasibility} is identically zero as a function of $r_0$ and of every model parameter entering $T_\phi^{0,1}$.
Differentiation with respect to any such variable $a$ yields
\begin{equation}    D_a[R_h\circ G_{\phi,c}^{\mathrm{RF}}]
    =J_RJ_\chi C_c^{-1}D_aT_\phi^{0,1}=0.
\label{eq:app-composed-differential}
\end{equation}
Eq.~\ref{eq:app-composed-differential} shows that the encoded residual contributes no Gauss--Newton block on the valid
coordinate domain.
This local identity does not control the full endpoint Jacobian or global
optimization.

\begin{corollary}[Feasibility of a decoded intrinsic path]
\label{cor:app-path-feasibility}
Let $y:[0,1]\rightarrow\mathcal{Y}_c$ be any path and define $x(s)=\chi_c(y(s))$.
Under Assumption~\ref{ass:regular-manifold},
\begin{equation}    R_h(x(s),c)=0
    \qquad\text{for every }s\in[0,1].
\label{eq:app-path-feasibility}
\end{equation}
If $y$ is differentiable, then $\dot x(s)=J_\chi(y(s),c)\dot y(s)\in T_{x(s)}\mathcal{M}_c^\chi$.
\end{corollary}

\begin{proof}
Eq.~\ref{eq:app-exact-chart} gives the feasibility in Eq.~\ref{eq:app-path-feasibility}; differentiating the path and applying Lemma~\ref{lem:app-tangent-space} gives tangency.
\end{proof}

\subsection{Intrinsic-coordinate penalty block}
\label{app:normal-tangent-proof}
\label{app:penalty-quotient-proof}

For an ambient linear residual $R_h(x)=A_hx-b$, the penalized objective and the Hessian are
\begin{equation} \mathcal L_\lambda(x)=\mathcal L_{\mathrm{gen}}(x)
 +\frac{\lambda}{2}\norm{A_hx-b}_{W}^{2},\qquad
 \nabla^2\mathcal L_\lambda=H_{\mathrm{gen}}+\lambda A_h^\top WA_h.
 \label{eq:ambient-composite-curvature}
\end{equation}

\begin{theorem}[Penalty-block scaling under exact intrinsic coordinates]
\label{thm:penalty-quotient}
Write $H_{\lambda,h}=H_{\mathrm{gen}}+\lambda A_h^\top WA_h$.
Let $x=x_p+N_hy$ with $A_hx_p=b$, $A_hN_h=0$, and $N_h^\top N_h=I$.
Then the intrinsic Hessian is $N_h^\top H_{\mathrm{gen}}N_h$ and is independent of $\lambda$.
If $A_h$ discretizes an order-$p$ operator with $\sigma_{\max}(A_h)=\Theta(h^{-p})$, $\sigma_{\min}^+(A_h)=\Theta(1)$, and $W$ is uniformly spectrally equivalent to the identity on the active residual space, the largest eigenvalue contributed by the residual penalty block is $\Theta(\lambda h^{-2p})$.
For an ambient condition-number statement, fix a declared subspace $V_h$ and interpret restriction with an orthonormal basis of $V_h$.
Suppose $H_{\mathrm{gen}}\succeq0$, the restriction $H_{\lambda,h}|_{V_h}$ is positive definite, $\sigma_{\max}(A_h|_{V_h})=\Theta(h^{-p})$, and $V_h\cap\mathop{\mathrm{ker}}\nolimits A_h$ contains a unit vector $z_h$ with $0<z_h^\top H_{\mathrm{gen}}z_h\leq L_{\mathrm{tan}}$ for an $h$- and $\lambda$-independent constant $L_{\mathrm{tan}}$.
Then
\begin{equation} \kappa(H_{\lambda,h}|_{V_h})
 \geq \frac{c_W\lambda\sigma_{\max}(A_h|_{V_h})^2}{L_{\mathrm{tan}}}
 =\Omega(\lambda h^{-2p}),
 \label{eq:ambient-condition-lower-bound}
\end{equation}
where $c_WI\preceq W$ on the active residual space.
Writing $\kappa_+$ for the ratio of the largest to the smallest positive eigenvalue, if a tangent preconditioner $S$ satisfies
\begin{equation} (1-\delta)I\preceq S^\top N_h^\top H_{\mathrm{gen}}N_hS\preceq(1+\delta)I,
 \qquad 0\leq\delta<1,
 \label{eq:tangent-spectral-equivalence}
\end{equation}
then
\begin{equation} \kappa_+\!\left(S^\top N_h^\top H_{\mathrm{gen}}N_hS\right)
 \leq\frac{1+\delta}{1-\delta},
 \label{eq:normal-tangent-condition-bound}
\end{equation}
independently of $\lambda$.
\end{theorem}

Substituting $x=x_p+N_hy$ into Eq.~\ref{eq:ambient-composite-curvature} gives $A_hx-b=A_hN_hy=0$.
Differentiating the restricted generative objective twice gives the intrinsic Hessian $N_h^\top H_{\mathrm{gen}}N_h$; the residual term contributes $\lambda N_h^\top A_h^\top WA_hN_h=0$ for every $\lambda$.

Uniform spectral equivalence of $W$ implies $0<c_W\leq C_W<\infty$, independent of $h$, such that
\begin{equation*} c_W\|A_hz\|_2^2\leq z^\top A_h^\top WA_hz\leq C_W\|A_hz\|_2^2.
\end{equation*}
On the active normal subspace, the assumed singular-value scaling gives
\begin{equation*} \lambda_{\max}(\lambda A_h^\top WA_h)=\Theta(\lambda h^{-2p}),
 \qquad
 \lambda_{\min}^{+}(\lambda A_h^\top WA_h)=\Theta(\lambda).
\end{equation*}
This establishes the claimed normal stiffness scaling.
Under the additional hypotheses of Theorem~\ref{thm:penalty-quotient}, positive semidefiniteness gives
\begin{equation*} \lambda_{\max}(H_{\lambda,h}|_{V_h})
 \geq c_W\lambda\sigma_{\max}(A_h|_{V_h})^2.
\end{equation*}
For the unit vector $z_h\in V_h\cap\mathop{\mathrm{ker}}\nolimits A_h$ in the theorem,
\begin{equation*} \lambda_{\min}(H_{\lambda,h}|_{V_h})
 \leq z_h^\top H_{\lambda,h}z_h
 =z_h^\top H_{\mathrm{gen}}z_h\leq L_{\mathrm{tan}}.
\end{equation*}
Dividing proves Eq.~\ref{eq:ambient-condition-lower-bound}.
The ambient bound applies on the declared subspace and does not remove
generator-dependent tangential curvature.
Finally, the Loewner bounds in Eq.~\ref{eq:tangent-spectral-equivalence} place every active eigenvalue of the preconditioned intrinsic Hessian in $[1-\delta,1+\delta]$; taking their ratio proves Eq.~\ref{eq:normal-tangent-condition-bound}.
The matrix $S$ is an abstract reference transform satisfying Eq.~\ref{eq:tangent-spectral-equivalence}; the implemented transforms $C_c$ and $P_{s,c}$ are analyzed separately in Proposition~\ref{prop:app-local-pe-conditioning}.

For the ideal residual-space Riesz choice $W_h=(A_hA_h^\top+\varepsilon I)^{-1}$, the nonzero eigenvalues of $A_h^\top W_hA_h$ are $\sigma_i^2/(\sigma_i^2+\varepsilon)$.
They approach one when $\varepsilon$ is small relative to the active spectrum.
For small $\varepsilon$ relative to the active spectrum, this balances the residual-normal block, while intrinsic parameterization and tangent conditioning control the remaining feasible directions.

\subsection{Approximate parameterizations and error bounds}
\label{app:approximate-chart-proof}

For an approximate decoder, define two independent approximation errors
\begin{equation}\begin{aligned}
    \varepsilon_{\mathrm{chart}}
    &\coloneqq
    \sup_{y\in\mathcal{Y}_c}
    \norm{R_h(\chi_c(y),c)}_2,\\
    \delta_{\mathrm{chart}}
    &\coloneqq
    \sup_{y\in\mathcal{Y}_c}
    \norm{D_y[R_h(\chi_c(y),c)]}_2.
\end{aligned}
    \label{eq:app-two-chart-errors}
\end{equation}

where $\varepsilon_{\mathrm{chart}}$ controls residual value, whereas
$\delta_{\mathrm{chart}}$ controls the pulled-back residual differential.

\begin{theorem}[Approximate feasibility and residual Gauss--Newton-block bound]
\label{thm:app-approximate-chart}
Suppose the quantities in Eq.~\ref{eq:app-two-chart-errors} are finite and
$T_\phi^{0,1}(r_0,c)$ maps every evaluated source coordinate into the coordinate domain.
For $G_{\phi,c}^{\mathrm{RF}}=\Psi_c\circ T_\phi^{0,1}$,
\begin{equation}    \norm{R_h(G_{\phi,c}^{\mathrm{RF}}(r_0,c),c)}_2
    \leq\varepsilon_{\mathrm{chart}}.
    \label{eq:app-approx-feasibility}
\end{equation}
For any differentiable variable $a$ and $W_R\succeq0$,
\begin{equation}    \norm{H_{\mathrm{GN},a}^{R}}_2
    \leq
    \norm{W_R}_2\,
    \delta_{\mathrm{chart}}^2
    \norm{C_c^{-1}D_aT_\phi^{0,1}}_2^2.
    \label{eq:app-approx-gn-bound}
\end{equation}
\end{theorem}

\begin{proof}
Eq.~\ref{eq:app-approx-feasibility} follows from the first supremum in Eq.~\ref{eq:app-two-chart-errors}.
The chain rule gives
\begin{equation*}    D_a[R_h\circ\Psi_c\circ T_\phi^{0,1}]
    =D_y[R_h\circ\chi_c]C_c^{-1}D_aT_\phi^{0,1}.
\end{equation*}
The operator norm is at most
$\delta_{\mathrm{chart}}\norm{C_c^{-1}D_aT_\phi^{0,1}}_2$.
Applying $\norm{A^\top W_RA}_2\leq\norm{W_R}_2\norm{A}_2^2$ proves Eq.~\ref{eq:app-approx-gn-bound}.
\end{proof}

The two errors are independent: a small residual value need not imply a small
pulled-back residual differential.
Both quantities are needed to characterize approximate decoders.
Under differentiability, exact parameterization corresponds to
$\varepsilon_{\mathrm{chart}}=\delta_{\mathrm{chart}}=0$.

\subsection{Co-area law}
\label{app:coarea-law}

For a density $\rho(x\mid c)$ relative to $dV_{M_h}$, hard conditioning by a full-row-rank residual is defined by a co-area target~\citep{xu2026coarea}.
On an injective chart, with $G_c(y)=J_\chi(y,c)^\top M_hJ_\chi(y,c)$ and $J_R=D_xR_h\vert_{\chi_c(y)}$, the intrinsic density is
\begin{equation} \pi_Y^{\mathrm{coa}}(y\mid c)
 =\frac{\rho(\chi_c(y)\mid c)}{Z_c}
 \frac{\sqrt{\det G_c(y)}}
      {\sqrt{\det\!\left(J_R(M_h)^{-1}J_R^\top\right)}}.
 \label{eq:coarea-coordinate-law}
\end{equation}
where $Z_c$ normalizes the density; the numerator is the tangent volume element and the denominator is the residual-normal co-area Jacobian in the same metric.

\begin{proposition}[Coordinate invariance of the co-area target]
\label{prop:coarea-chart-invariance}
Let $\chi:\mathcal Y\to\mathcal M^\chi$ satisfy the hypotheses of Eq.~\ref{eq:coarea-coordinate-law}, let $\varphi:\widetilde{\mathcal Y}\to\mathcal Y$ be a $C^1$ diffeomorphism, and set $\widetilde\chi=\chi\circ\varphi$.
If $\pi_Y^{\mathrm{coa}}$ is the density in Eq.~\ref{eq:coarea-coordinate-law}, then the density computed from $\widetilde\chi$ obeys
\begin{equation} \pi_{\widetilde Y}^{\mathrm{coa}}(\widetilde y)
 =\pi_Y^{\mathrm{coa}}(\varphi(\widetilde y))
   \left|\det D\varphi(\widetilde y)\right|,
 \label{eq:coarea-chart-change}
\end{equation}
The pushforward measures coincide: $\widetilde\chi_\#(\pi_{\widetilde Y}^{\mathrm{coa}}\,d\widetilde y)=\chi_\#(\pi_Y^{\mathrm{coa}}\,dy)$.
Reparameterization changes coordinate density but preserves the induced physical measure on the feasible component.
\end{proposition}

\begin{proof}
The chain rule gives $J_{\widetilde\chi}=J_\chi D\varphi$ and yields
\begin{equation*} \det(J_{\widetilde\chi}^\top M_xJ_{\widetilde\chi})
 =\det(D\varphi)^2\det(J_\chi^\top M_xJ_\chi).
\end{equation*}
The ambient density and residual-normal determinant are evaluated at the same physical point $\chi(\varphi(\widetilde y))$.
Substitution into Eq.~\ref{eq:coarea-coordinate-law} yields Eq.~\ref{eq:coarea-chart-change}; the ordinary change-of-variables theorem then proves equality of the pushforward measures.
\end{proof}

\subsection{Encoding-induced representation error}
\label{app:encoding-error}

For data outside the parameterization range, assume that the minimum below
is attained and choose a measurable encoder selection
\begin{equation} \eta_c(x)\in\arg\min_{y\in\mathcal Y_c}
 \norm{\chi_c(y)-x}_{M_h}^2.
\label{eq:declared-encoder}
\end{equation}

For the encoder in Eq.~\ref{eq:declared-encoder}, let $X\sim p_{\mathrm{data}}(\cdot\mid c)$ and define the feasible projection $\bar X=\chi_c(\eta_c(X))$.
The joint law of $(X,\bar X)$ is a valid coupling of $p_{\mathrm{data}}$ and $\bar p_{\mathrm{data}}=(\chi_c\circ\eta_c)_\#p_{\mathrm{data}}$.
By the definition of the Wasserstein distance,
\begin{equation}    W_{2,M_x}^2(\bar p_{\mathrm{data}},p_{\mathrm{data}})
    \leq
    \mathbb{E}\norm{\bar X-X}_{M_x}^2.
\label{eq:app-wasserstein-projection}
\end{equation}
Eq.~\ref{eq:app-wasserstein-projection} bounds the representation error; we track it separately from generative error.

\section{Support, Measure, and Transport Errors}
\label{app:support-measure-transport}

Residual factorization fixes feasible support, and the target-law construction
specifies probability mass; matching determines how closely the learned
pushforward follows that law.
The following bound separates transport, measure, and representation error.

\begin{proposition}[Three-way error bookkeeping bound]
\label{prop:three-way-error}
Let $\widehat\pi$, $\pi_{\mathrm{enc}}$, $\pi_{\mathrm{law}}$, and $\pi_{\mathrm{data}}$ be probability measures with finite second moments under the metric induced by $M_x$.
Then
\begin{equation} W_{2,M_x}(\widehat\pi,\pi_{\mathrm{data}})
 \leq W_{2,M_x}(\widehat\pi,\pi_{\mathrm{enc}})
 +W_{2,M_x}(\pi_{\mathrm{enc}},\pi_{\mathrm{law}})
 +W_{2,M_x}(\pi_{\mathrm{law}},\pi_{\mathrm{data}}).
 \label{eq:app-three-way-error}
\end{equation}
\end{proposition}

\begin{proof}
Apply the triangle inequality for $W_{2,M_x}$ first through $\pi_{\mathrm{enc}}$ and then through $\pi_{\mathrm{law}}$.
\end{proof}

Eq.~\ref{eq:app-three-way-error} separates transport, target-law, and representation error.
The latter two vanish for exactly feasible sample-defined targets; density-defined targets use the co-area law in Eq.~\ref{eq:coarea-coordinate-law}.

\section{Physical and Statistical Preconditioning}
\label{app:preconditioning-theory}

After residual factorization, $C_c$ rescales the decoder-induced physical
metric in Eq.~\ref{eq:pullback-physical-metric}, while $P_{s,c}$ conditions the interpolation statistics seen by the
network.
We analyze their effects on local regression conditioning separately.
Writing $J_{\Psi_c}=D_r\Psi_c$ and $y=\mu_c+C_c^{-1}r$, the chain rule gives
$J_{\Psi_c}(r)=J_\chi(y,c)C_c^{-1}$, and the transformed metric is
$\widetilde G_c(r)=J_{\Psi_c}(r)^\top M_hJ_{\Psi_c}(r)$.

\subsection{Regularized input covariance}

\begin{lemma}[Spectrum of regularized whitening]
\label{lem:app-regularized-whitening}
Let $C\succeq0$ and $P=(C+\varepsilon I)^{-1/2}$ with $\varepsilon>0$.
If $C=U\mathop{\mathrm{diag}}\nolimits(\lambda_i)U^\top$, then
\begin{equation}    PCP^\top
    =U\mathop{\mathrm{diag}}\nolimits\!\left(
      \frac{\lambda_i}{\lambda_i+\varepsilon}
    \right)U^\top.
\label{eq:app-regularized-spectrum}
\end{equation}
On the active subspace of $C$,
\begin{equation}    \kappa_+(PCP^\top)
    =
    \frac{\lambda_{\max}(C)/(\lambda_{\max}(C)+\varepsilon)}
         {\lambda_{\min}^+(C)/(\lambda_{\min}^+(C)+\varepsilon)}.
\label{eq:app-regularized-condition}
\end{equation}
Exact whitening is recovered when $\varepsilon=0$ and $C$ is nonsingular.
\end{lemma}

\begin{proof}
Applying $(\lambda+\varepsilon)^{-1/2}$ to the eigendecomposition of $C$ gives Eq.~\ref{eq:app-regularized-spectrum} and the condition number in Eq.~\ref{eq:app-regularized-condition}.
\end{proof}

For the interpolation states in Eq.~\ref{eq:interpolation-pair}, the transform in Eq.~\ref{eq:input-preconditioner} gives
\begin{equation} P_{s,c}\Sigma_{s,c}P_{s,c}^\top
 =\Sigma_{s,c}(\Sigma_{s,c}+\varepsilon_P I)^{-1}.
\label{eq:app-state-preconditioner}
\end{equation}
Eq.~\ref{eq:app-state-preconditioner} shows how $P_{s,c}$ conditions network inputs while preserving the specified physical target law.

\subsection{Local physical-space conditioning}

Fix $s<t$, $c$, and a stopped-gradient target branch.
Let $\zeta=P_{s,c}(r_s-m_{s,c})$ and consider a local linear velocity model $\bar u_A=A\zeta$.
The endpoint depending on $A$ is $a_A=r_s+(t-s)A\zeta$.
\begin{proposition}[Local decoded-endpoint conditioning]
\label{prop:app-local-pe-conditioning}
The Gauss--Newton matrix of the decoded-endpoint loss with respect to $A$ is
\begin{equation} H_A^{\mathrm{GN}}
 =\frac{t-s}{m\delta}\,
 \mathbb E\left[(\zeta\zeta^\top)\otimes\widetilde G_c(a_A)\right].
\label{eq:app-local-pe-hessian}
\end{equation}
Assume $\alpha I\preceq\widetilde G_c(a_A)\preceq\beta I$ in the stated region and let $S_{s,c}=\mathbb E[\zeta\zeta^\top]$ be positive definite.
Then
\begin{equation} \kappa\!\left(H_A^{\mathrm{GN}}\right)
 \leq\frac{\beta}{\alpha}\kappa(S_{s,c}),
 \qquad
 S_{s,c}=P_{s,c}\Sigma_{s,c}P_{s,c}^\top
\label{eq:app-local-pe-condition}
\end{equation}
under exact centering.
\end{proposition}

\begin{proof}
Differentiating the squared decoded endpoint error gives the pullback and input factors in Eq.~\ref{eq:app-local-pe-hessian}.
The stated spectral bounds sandwich the expectation between positive multiples of $S_{s,c}\otimes I$.
The extremal eigenvalues give Eq.~\ref{eq:app-local-pe-condition}.
\end{proof}

\section{One-Step Pullback Geometry}
\label{app:pullback-theory}

The pullback matrix $G_c(y)=J_\chi(y,c)^\top M_xJ_\chi(y,c)$ is exact for infinitesimal perturbations.
The action in preconditioned coordinates can be computed without forming a dense Jacobian:
\begin{equation} \widetilde G_c(r)v
 =J_{\Psi_c}(r)^\top
 \bigl[M_h\bigl(J_{\Psi_c}(r)v\bigr)\bigr].
\label{eq:matrix-free-pullback}
\end{equation}
Eq.~\ref{eq:matrix-free-pullback} uses a decoder Jacobian-vector product followed by a vector-Jacobian product.
The following result quantifies the approximation error for a finite coordinate displacement.

\begin{proposition}[Finite-displacement pullback error]
\label{prop:app-pullback-remainder}
Fix $y$ and a displacement $d$ such that the segment $y+sd$, $s\in[0,1]$, lies in $\mathcal{Y}_c$.
Freeze the positive-definite physical metric $M_x$.
Suppose
\begin{equation}    \norm{D_y^2\chi_c(y+sd)[v,v]}_{M_x}
    \leq L_\chi\norm{v}_2^2
\label{eq:app-chart-hessian-bound}
\end{equation}
for every $s\in[0,1]$ and $v$.
Let $K_\chi=\norm{J_\chi(y,c)}_{2\rightarrow M_x}$.
Then
\begin{equation}\begin{aligned}
    \big|
    \norm{\chi_c(y+d)-\chi_c(y)}_{M_x}^2
    -d^\top G_c(y)d
    \big|
    \leq
    L_\chi K_\chi\norm{d}_2^3
    +\frac{L_\chi^2}{4}\norm{d}_2^4.
\end{aligned}
    \label{eq:app-pullback-remainder}
\end{equation}
\end{proposition}

\begin{proof}
Taylor's theorem under the bound in Eq.~\ref{eq:app-chart-hessian-bound} gives
\begin{equation*}    \chi_c(y+d)-\chi_c(y)=J_\chi(y,c)d+r_d,
    \qquad
    \norm{r_d}_{M_x}\leq\frac{L_\chi}{2}\norm{d}_2^2.
\end{equation*}
Expanding the squared norm and applying Cauchy--Schwarz yields
\begin{equation*}    2\norm{J_\chi d}_{M_x}\norm{r_d}_{M_x}
    +\norm{r_d}_{M_x}^2,
\end{equation*}
which is bounded by the right-hand side of Eq.~\ref{eq:app-pullback-remainder}.
\end{proof}

The cubic remainder shows that the pullback metric is a local approximation,
whereas decoded endpoint error measures finite physical discrepancy directly.

\section{Feasible Parameterizations and Preconditioning}
\label{app:chart-examples}

\subsection{Linear null-space parameterization}

\begin{proposition}[Exact encoding for affine constraints]
\label{prop:app-nullspace-chart}
Let $R_h(x)=Ax-b$, choose $x_p$ with $Ax_p=b$, and let the columns of $N$ form a basis of $\mathop{\mathrm{ker}}\nolimits A$.
Then
\begin{equation}    \chi(y)=x_p+Ny
    \label{eq:app-nullspace-chart}
\end{equation}
is an exact global constraint-satisfying parameterization of the affine feasible set.
Evaluating the constraint yields
\begin{equation}    R_h(\chi(y))=0,
    \qquad
    J_RJ_\chi=AN=0.
\label{eq:app-nullspace-annihilation}
\end{equation}
\end{proposition}

\begin{proof}
Substituting Eq.~\ref{eq:app-nullspace-chart} into $Ax-b$ gives Eq.~\ref{eq:app-nullspace-annihilation}.
\end{proof}

This construction covers discrete conservation, boundary, and observation constraints whenever they can be written as a compatible affine system.
Numerical quality depends on the representation of the null-space basis; the exact algebraic identity and coordinate conditioning are separate properties.

\subsection{Discrete streamfunction parameterization}

Let $D_x$ and $D_y$ be discrete derivative matrices on the same grid and with the same boundary convention.
Define
\begin{equation}    \chi(\psi)
    =\begin{bmatrix}D_y\psi\\-D_x\psi\end{bmatrix}.
    \label{eq:app-streamfunction-chart}
\end{equation}

\begin{proposition}[Compatible discrete incompressibility]
\label{prop:app-incompressibility}
If $D_xD_y=D_yD_x$, then every velocity decoded by Eq.~\ref{eq:app-streamfunction-chart} satisfies the matching discrete divergence exactly:
\begin{equation}    D_xu+D_yv
    =(D_xD_y-D_yD_x)\psi=0.
\label{eq:app-discrete-divergence}
\end{equation}
For the Euclidean state metric, the pullback matrix is
\begin{equation}    G_\psi=D_y^\top D_y+D_x^\top D_x.
    \label{eq:app-streamfunction-pullback}
\end{equation}
\end{proposition}

\begin{proof}
Commutation gives the residual identity in Eq.~\ref{eq:app-discrete-divergence}.
Eq.~\ref{eq:app-streamfunction-pullback} follows by multiplying the Jacobian of Eq.~\ref{eq:app-streamfunction-chart} by the corresponding transpose factors.
\end{proof}

Eq.~\ref{eq:app-streamfunction-pullback} separates exact incompressibility from intrinsic metric geometry: the parameterization enforces the former, while $G_\psi$ determines the latter.
On a periodic grid, the Fourier symbol of the pullback matrix scales like the squared discrete wave number.
The pullback metric makes the physical importance of gradients and small scales explicit.

\section{Physical-Space Endpoint Matching}
\label{app:match-theory}

The two-time map in Eq.~\ref{eq:two-time-map} is supervised at $s=t$ by the velocity loss in Eq.~\ref{eq:diagonal-fm-loss}.
At $s<t$, the loss in Eq.~\ref{eq:physical-endpoint-loss} compares two estimates of the same decoded physical endpoint.
The following result links the learned finite-interval map to endpoint distribution error.

Let $Y_s$ follow a target marginal path with $\dot Y_s=v_s(Y_s,c)$.
Let
\[
p_{\mathrm{enc},c}=\mathcal L(\Psi_c(Y_1))
\]
denote the encoded target endpoint law.
For the learned map, define the coordinate and physical defects
\begin{equation} e_\theta(s,r)
 =\partial_sT_\theta^{s,1}(r,c)
 +D_rT_\theta^{s,1}(r,c)v_s(r,c),
 \qquad
 d_\theta(s,r)=J_{\Psi_c}\!\left(T_\theta^{s,1}(r,c)\right)e_\theta(s,r).
\label{eq:app-flow-map-defect}
\end{equation}
\begin{proposition}[Physical flow-map defect bound]
\label{prop:app-flow-map-defect}
Assume a smooth target path, a valid coordinate domain, finite second moments,
and
\[
\mathbb E\!\int_0^1
\|d_\theta(s,Y_s)\|_{M_x}^2\,ds<\infty.
\]
For the defects in Eq.~\ref{eq:app-flow-map-defect}, the endpoint law $\widehat p_c$ of $\Psi_c(T_\theta^{0,1}(Y_0,c))$ satisfies
\begin{equation} W_{2,M_x}^2\!\left(\widehat p_c,p_{\mathrm{enc},c}\right)
 \leq\int_0^1\mathbb E\left[
 e_\theta(s,Y_s)^\top
 \widetilde G_c\!\left(T_\theta^{s,1}(Y_s,c)\right)
 e_\theta(s,Y_s)
 \right]ds.
 \label{eq:app-flow-map-bound}
\end{equation}
\end{proposition}

\begin{proof}
Along the target path,
\begin{equation*} \Psi_c(Y_1)-\Psi_c\!\left(T_\theta^{0,1}(Y_0,c)\right)
 =\int_0^1d_\theta(s,Y_s)\,ds.
\end{equation*}
The common source $Y_0$ couples the endpoint laws; Cauchy--Schwarz and $\widetilde G_c$ give Eq.~\ref{eq:app-flow-map-bound}.
\end{proof}

The consistency loss is a finite-difference surrogate for the flow-map defect
controlled by Eq.~\ref{eq:app-flow-map-bound}.
For the interpolation path, let
\[
v_s(r,c)=\mathbb E[w\mid r_s=r,c].
\]
For fixed $s,t,c,\delta$ and a stopped-gradient target, conditioning on $r_s$
and averaging over the endpoint-pair velocity $w$ gives the corresponding
conditional semi-gradient with respect to the differentiable velocity output:
\begin{equation} -\frac1m\widetilde G_c(a_\theta)
 \left(
 \partial_sT_\theta^{s,t}+D_rT_\theta^{s,t}v_s
 \right)+O(\delta).
\label{eq:app-semigradient-limit}
\end{equation}
Eq.~\ref{eq:app-semigradient-limit} fixes the target branch and uses $\delta$ to normalize the finite difference.

\section{Regular Residual Geometries}
\label{app:general-residual-geometries}

For regular residual manifolds, PMosFM constructs local feasible coordinates from discrete residual-zero sets rather than assuming a prescribed Riemannian manifold~\citep{chen2023flow}.

\subsection{Local residual-manifold parameterizations}
\label{app:general-residual-coordinates}

Fix the mesh $h$ and condition $c$, and let $M_x=M_h$ be positive definite. Assume that $R_h(\cdot,c)$ is $C^2$ on an open neighborhood and that $D_xR_h$ has constant rank $q$ there. On the regular component
\begin{equation*}\mathcal M_c=\{x:R_h(x,c)=0\},\qquad m=d-q,
\end{equation*}
the tangent space is $T_x\mathcal M_c=\mathop{\mathrm{ker}}\nolimits D_xR_h(x,c)$, with physical metric $g_x(\xi,\zeta)=\xi^\top M_x\zeta$.

\begin{proposition}[Local feasible parameterizations]
\label{prop:general-residual-charts}
Every point of a regular component has a local injective parameterization $\chi_c:\mathcal Y_c\to\mathcal M_c$ with $R_h\circ\chi_c=0$. For a compact regular target region, finitely many such parameterizations cover the region. After an invertible coordinate transform $C_{j,c}$, each decoder
\begin{equation*}\Psi_{j,c}(r)=\chi_{j,c}(\mu_{j,c}+C_{j,c}^{-1}r)
\end{equation*}
remains exactly feasible and has continuous pullback metric $\widetilde G_{j,c}(r)=D\Psi_{j,c}(r)^\top M_xD\Psi_{j,c}(r)$.
\end{proposition}

A near-identity metric bound for $C_{j,c}$ requires separate verification; the construction is local.

The constructions above apply patchwise. Let $\Sigma_{s,j,c}$ denote the
covariance of $r_s$ within patch $j$, and define
\begin{equation*}P_{s,j,c}
=(\Sigma_{s,j,c}+\varepsilon_P I)^{-1/2},
\qquad
\widetilde G_{j,c}(r)
=D\Psi_{j,c}(r)^\top M_xD\Psi_{j,c}(r).
\end{equation*}
Under the assumptions of Proposition~\ref{prop:app-local-pe-conditioning},
if $\alpha_{j,c}I\preceq\widetilde G_{j,c}(r)\preceq\beta_{j,c}I$ on the
evaluated region, then
\begin{equation}\kappa(H_A^{\mathrm{GN}})
\leq
\frac{\beta_{j,c}}{\alpha_{j,c}}
\kappa\!\left(
P_{s,j,c}\Sigma_{s,j,c}P_{s,j,c}^{\top}
\right).
\label{eq:general-residual-conditioning}
\end{equation}

For overlapping patches, let $a_j$ form a partition of unity on the target
support and set $\omega_j=\int a_j\,d\nu_c$, so $\sum_j\omega_j=1$.
For $\omega_j>0$, define the normalized patch law
\[
\nu_{j,c}=\frac{a_j\nu_c}{\omega_j},
\]
and let $\widehat\nu_{j,c}$ denote the corresponding generated endpoint law.
Set
\[
\widehat\nu_c=\sum_j\omega_j\widehat\nu_{j,c}.
\]
Applying Proposition~\ref{prop:app-flow-map-defect} within each patch and
mixing the resulting couplings gives
\begin{equation}W_{2,M_x}^{2}(\widehat\nu_c,\nu_c)
\leq
\sum_j\omega_j\mathcal E_{j,c},
\label{eq:general-residual-endpoint-bound}
\end{equation}
where $\mathcal E_{j,c}$ is the integrated physical flow-map defect from
Eq.~\ref{eq:app-flow-map-bound}, evaluated on patch $j$.
Eqs.~\ref{eq:general-residual-conditioning} and~\ref{eq:general-residual-endpoint-bound} extend the conditioning and endpoint-error bounds locally without a global flattening of the residual manifold.

\section{Additional Analyses}
\label{app:additional-samples-analysis}

Supplementary evaluations follow the dataset partitions, units, and evaluators in Appendix~\ref{app:pde-dataset-details}.

\subsection{Feasibility, Distribution, and Transport}
\label{app:controlled-smt}

\begin{wraptable}[8]{r}{0.3\textwidth}
\centering
\caption{Normal conditioning.}
\label{tab:controlled-mesh-conditioning}
\scriptsize
\setlength{\tabcolsep}{2.5pt}
\renewcommand{\arraystretch}{1.05}
\begin{tabular}{@{}ccc@{}}
\toprule
$N$ & $\kappa(A_h^\top A_h)$ & $\kappa(A_h^\top W_hA_h)$ \\
\midrule
32  & $4.33\times10^{4}$ & $1.0001$ \\
64  & $6.90\times10^{5}$ & $1.0001$ \\
128 & $1.10\times10^{7}$ & $1.0001$ \\
256 & $1.76\times10^{8}$ & $1.0001$ \\
\bottomrule
\end{tabular}
\end{wraptable}

Appendices~\ref{app:controlled-support-measure}, \ref{app:chart-diagnostics}, and \ref{app:controlled-conditioning} describe support, decoder, and conditioning diagnostics.
Fig.~\ref{fig:residual-manifold} shows that exact parameterization places
samples on the feasible set, while the volume-only law still differs
from the target density; the co-area correction reduces this mismatch.

\begin{figure}[t]
    \centering
    \includegraphics[width=0.97\columnwidth]{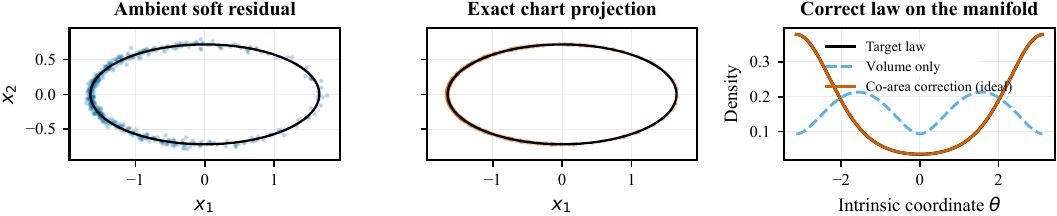}
    \caption{Support and measure on a residual manifold. Exact parameterization enforces feasibility, and the co-area correction recovers the target law.}
    \label{fig:residual-manifold}
\end{figure}

\begin{figure*}[t]
 \centering
 \includegraphics[width=0.97\textwidth]{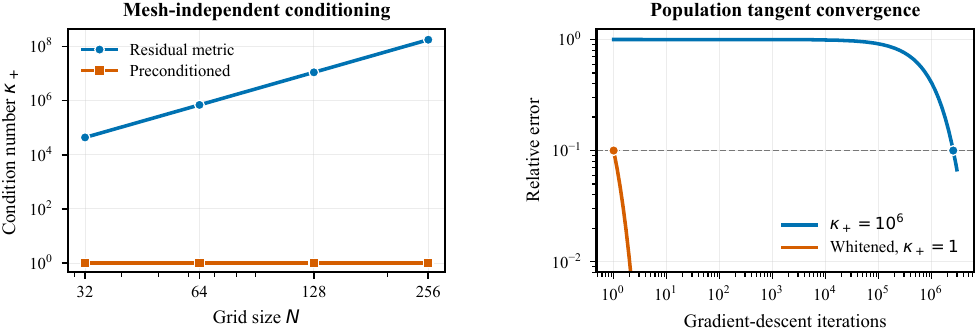}
 \caption{Normal and tangent conditioning.
Left: preconditioning controls residual-normal conditioning under grid refinement.
Right: whitening accelerates tangent-regression convergence.}
 \label{fig:preconditioning-convergence}
\end{figure*}

\subsection{Decoder tolerance and cost}
\label{app:chart-tolerance}

\paragraph{Decoder formulation.}
For a source coordinate $r_0$, PMosFM first evaluates the learned finite-interval
map and then applies the geometric inverse transform:
\begin{equation}    \widehat r_1=T_\theta^{0,1}(r_0,c),
    \qquad
    \widehat y_1=\mu_c+C_c^{-1}\widehat r_1,
    \qquad
    \widehat x_1^{(K)}=\chi_c^{(K)}(\widehat y_1).
\label{eq:app-decoder-endpoint}
\end{equation}
In Eq.~\ref{eq:app-decoder-endpoint}, $K$ indexes the physical decoder, not the learned transport map; for an explicit constraint-satisfying parameterization, decoding applies the algebraic or compatible discrete differential operations defining $\chi_c$, covering the affine, algebraic, and compatible potential-based cases considered in the experiments:
\begin{equation*}    \widehat x_1=\chi_c(\widehat y_1),
    \qquad
    R_h(\chi_c(y),c)\equiv0,
\end{equation*}

For an implicit parameterization, write $x=(y,z)$. Here $z^\star(y,c)$ denotes the selected solution branch, whereas $\mathcal S_h$ denotes one decoder iteration:
\begin{equation}    F_h(z^\star,y,c)=0,
    \qquad
    z^{(k+1)}=\mathcal S_h(z^{(k)};y,c),
    \qquad
    \chi_c^{(K)}(y)=
    \begin{bmatrix}y\\z^{(K)}(y,c)\end{bmatrix}.
\label{eq:app-solver-decoder}
\end{equation}
The Darcy implementation uses the decoder in Eq.~\ref{eq:app-solver-decoder}, with $K=256$ in this configuration.
For an exact implicit case $F_h(z,y,c)=0$ with invertible $D_zF_h$, the directional derivative satisfies
\begin{equation*} D_yz[v]=-(D_zF_h)^{-1}D_yF_h[v].
\end{equation*}
Finite-iteration decoders require separate derivative-accuracy checks.

\paragraph{Decoder accuracy and residual bounds.}
For a $K$-iteration decoder, define the value and differential diagnostics
\begin{equation}\begin{aligned}
    \varepsilon_{\mathrm{dec}}^{(K)}
    &:=\sup_{y\in\mathcal Y_c}
    \norm{R_h(\chi_c^{(K)}(y),c)}_2,\\
    \delta_{\mathrm{dec}}^{(K)}
    &:=\sup_{y\in\mathcal Y_c}
    \norm{D_y\!\left[R_h\circ\chi_c^{(K)}\right](y,c)}_2.
\end{aligned}
\label{eq:app-decoder-certificates}
\end{equation}
In Eq.~\ref{eq:app-decoder-certificates}, the first diagnostic measures physical residual, and the second measures its differential in the represented coordinates.
With
\begin{equation*}    G_{\theta,c}^{(K)}
    =\chi_c^{(K)}\circ
    \left(\mu_c+C_c^{-1}\,\cdot\right)\circ T_\theta^{0,1},
\end{equation*}
the residual bound gives
\begin{equation*}    \norm{R_h\!\left(G_{\theta,c}^{(K)}(r_0,c),c\right)}_2
    \leq\varepsilon_{\mathrm{dec}}^{(K)}.
\end{equation*}
For any differentiable source or model variable $a$,
\begin{equation*}    \norm{D_a\!\left[R_h\circ G_{\theta,c}^{(K)}\right]}_2
    \leq
    \delta_{\mathrm{dec}}^{(K)}
    \norm{C_c^{-1}D_aT_\theta^{0,1}}_2.
\end{equation*}
The two diagnostics address different properties: residual magnitude and
the residual-normal differential after approximate decoding.
The corresponding residual Gauss--Newton block satisfies
\begin{equation}\begin{aligned}
    H_{\mathrm{GN},a}^{R,(K)}
    &:=\left(D_a[R_h\circ G_{\theta,c}^{(K)}]\right)^\top
    W_R\,D_a[R_h\circ G_{\theta,c}^{(K)}],\\
    \norm{H_{\mathrm{GN},a}^{R,(K)}}_2
    &\leq \norm{W_R}_2\left(\delta_{\mathrm{dec}}^{(K)}\right)^2
    \norm{C_c^{-1}D_aT_\theta^{0,1}}_2^2.
\end{aligned}
\label{eq:app-decoder-gn-bound}
\end{equation}
For an exact decoder, $\varepsilon_{\mathrm{dec}}=\delta_{\mathrm{dec}}=0$ and
the encoded residual-normal block vanishes; Eq.~\ref{eq:app-decoder-gn-bound} bounds the block retained by an iterative decoder.

\paragraph{End-to-end cost.}
We distinguish learned transport from physical decoding:
\begin{equation}    T_{\mathrm{sample}}^{(K)}
    =T_{\mathrm{map}}+T_{\mathrm{pre}}+T_{\mathrm{dec}}^{(K)},
    \qquad
    \mathrm{NFE}_{\mathrm{learned}}=1,
\label{eq:app-sampling-cost}
\end{equation}
In Eq.~\ref{eq:app-sampling-cost}, $T_{\mathrm{map}}$ is the cost of $T_\theta^{0,1}$, $T_{\mathrm{pre}}$
contains the fixed coordinate and input transforms, and $T_{\mathrm{dec}}^{(K)}$
is the physical decoding cost.
For a direct setting $T_{\mathrm{dec}}^{\mathrm{dir}}=T_\chi$; for an iterative
decoder,
\begin{equation*}    T_{\mathrm{dec}}^{(K)}
    =T_{\mathrm{init}}+
    \sum_{k=0}^{K-1}T_{\mathcal S}^{(k)}+T_{\mathrm{post}}
    \approx T_{\mathrm{init}}+K\,\overline T_{\mathcal S}+T_{\mathrm{post}}.
\end{equation*}
The timing studies report learned transport and full physical decoding separately,
while the residual diagnostics quantify the accuracy of the realized PMosFM endpoint.

\end{document}

%% file: math_commands.tex
\usepackage{amsmath,amsfonts,bm}

\def\eqref#1{equation~\ref{#1}}

\def\1{\bm{1}}

\def\vmu{{\bm{\mu}}}

\def\vn{{\bm{n}}}

\def\vs{{\bm{s}}}

\def\vu{{\bm{u}}}

\def\vx{{\bm{x}}}
\def\vy{{\bm{y}}}

\DeclareMathAlphabet{\mathsfit}{\encodingdefault}{\sfdefault}{m}{sl}
\SetMathAlphabet{\mathsfit}{bold}{\encodingdefault}{\sfdefault}{bx}{n}